\documentclass{article}

\usepackage{microtype}
\usepackage{graphicx}
\usepackage{subcaption}
\usepackage{booktabs} % for professional tables
\usepackage{svg}
\usepackage{xifthen}

\usepackage{hyperref}

\usepackage[preprint]{icml2026}
\usepackage{amsmath}
\usepackage{amssymb}
\usepackage{mathtools}
\usepackage{amsthm}

\usepackage[capitalize,noabbrev]{cleveref}

\theoremstyle{plain}
\newtheorem{theorem}{Theorem}[section]
\newtheorem{proposition}[theorem]{Proposition}
\newtheorem{lemma}[theorem]{Lemma}
\newtheorem{corollary}[theorem]{Corollary}
\theoremstyle{definition}
\newtheorem{definition}[theorem]{Definition}

\theoremstyle{remark}
\newtheorem{remark}[theorem]{Remark}

\newcommand{\blue}[1]{\textcolor{blue}{#1}}

\newcommand{\inner}[2]{\langle #1, #2 \rangle}
\newcommand{\loss}{\mathcal{L}}
\renewcommand{\l}{\ell}

\icmltitlerunning{Why Deep Jacobian Spectra Separate}

\begin{document}

\twocolumn[
  \icmltitle{Why Deep Jacobian Spectra Separate: Depth-Induced Scaling and Singular-Vector Alignment}

  \icmlsetsymbol{equal}{*}

  \begin{icmlauthorlist}
    \icmlauthor{Nathana\"{e}l Haas}{cril,lmpa}
    \icmlauthor{Fran\c{c}ois Gatine}{imj}
    \icmlauthor{Augustin M Cosse}{lmpa}
    \icmlauthor{Zied Bouraoui}{cril}
  \end{icmlauthorlist}

  \icmlaffiliation{cril}{CRIL UMR 8188, Universit\'e d'Artois, CNRS, France}
  \icmlaffiliation{lmpa}{LMPA, Universit\'e du Littoral C\^{o}te d’Opale}
  \icmlaffiliation{imj}{IMJ-PRG, Sorbonne Universit\'e }

  \icmlcorrespondingauthor{Nathana\"{e}l Haas}{haas@cril.fr}
  \icmlcorrespondingauthor{Fran\c{c}ois Gatine}{francois.gatine@imj-prg.fr}

  \icmlkeywords{Machine Learning, ICML, Deep neural networks, Jacobian matrix, Lyapunov exponents, Spectral decay, Implicit bias}

  \vskip 0.3in
]


% command is standard text for equal contributionPreliminary. Remove it (just {}) if you

\printAffiliationsAndNotice{}  % no special notice (required even if empty)

%{\color{magenta}My recomendation: Hide all the complicated functions that appear in the statements as compact expressions (and provide the details in the appendix). Only highlight what is necessary. E.g. in the expression of the $\gamma_i$ you can write $\ln(\sqrt{2}\sigma ) + F(p,n)$ ``where $F(p,n)$ is a function that depends on ... but not on ..."}

\begin{abstract}
Understanding why gradient-based training in deep networks exhibits strong implicit bias remains challenging, in part because tractable singular-value dynamics are typically available only for balanced deep linear models. We propose an alternative route based on two theoretically grounded and empirically testable signatures of deep Jacobians: depth-induced exponential scaling of ordered singular values and strong spectral separation. Adopting a fixed-gates view of piecewise-linear networks, where Jacobians reduce to products of masked linear maps within a single activation region, we prove the existence of Lyapunov exponents governing the top singular values at initialization, give closed-form expressions in a tractable masked model, and quantify finite-depth corrections. We further show that sufficiently strong separation forces singular-vector alignment in matrix products, yielding an approximately shared singular basis for intermediate Jacobians. Together, these results motivate an approximation regime in which singular-value dynamics become effectively decoupled, mirroring classical balanced deep-linear analyses without requiring balancing. Experiments in fixed-gates settings validate the predicted scaling, alignment, and resulting dynamics, supporting a mechanistic account of emergent low-rank Jacobian structure as a driver of implicit bias.
\end{abstract}

\section{Introduction}
Deep networks often develop an effectively low-rank Jacobian geometry: a small number of singular directions dominate input–output sensitivity while the remainder stay weak. This behavior is typically accompanied by spectral separation, in the sense that ratios between consecutive ordered singular values increase, and it can sharpen during training; such anisotropy has been associated with generalization in realistic architectures \cite{oymak2019generalization}. Figure \ref{fig:intro} illustrates this phenomenon in a controlled fixed-gates linear network trained on a rank-10 synthetic task, where leading Jacobian singular values separate and evolve in a structured manner. These trajectories raise a mechanistic question: which depth-driven properties of Jacobian products create and maintain strong separation, and when do they enable tractable predictions for singular-value evolution?

\begin{figure}
    \centering
    \includegraphics[width=0.7\linewidth]{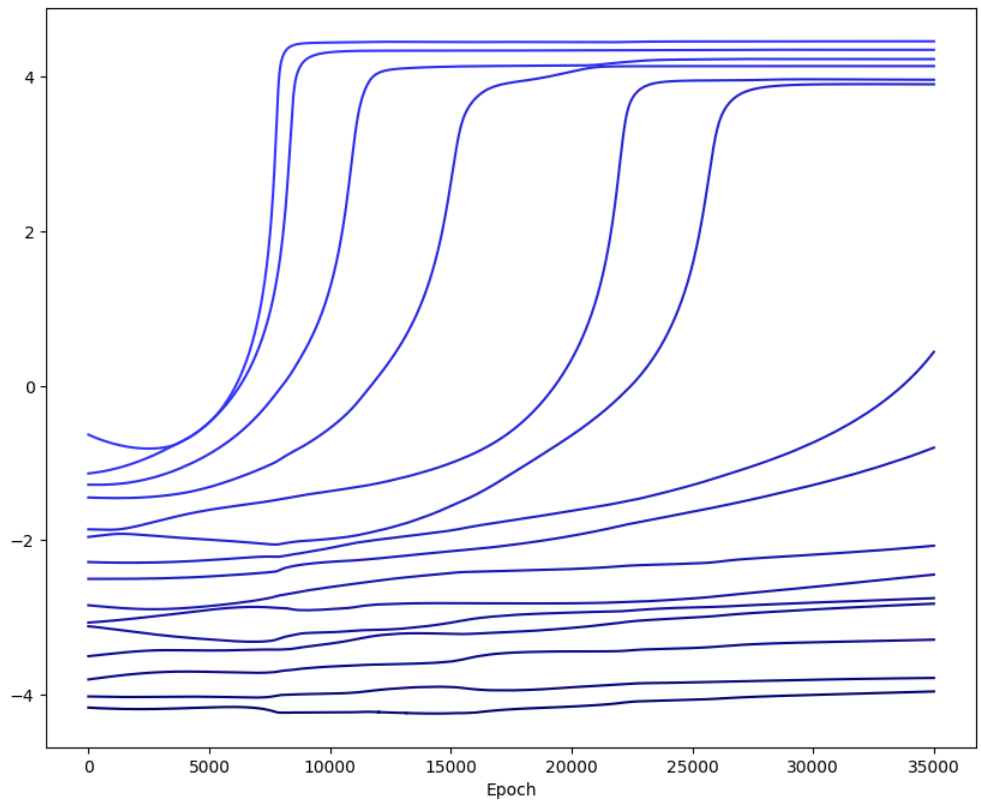}
    \caption{Evolution of the top-15 Jacobian log-singular values during training of a fixed-gates linear network (depth 10, width 64) on a synthetic rank-10 regression task. Inputs are Gaussian and targets are generated by a fixed random linear map of rank 10. Colors index order: brighter curves correspond to larger singular values (from $s_1$ to $s_{15}$).}
    \label{fig:intro}
\end{figure}

A classical setting where singular-value evolution is analytically tractable is deep linear networks and matrix factorization. Exact learning dynamics and depth-dependent amplification effects are well understood in the linear case \cite{saxe2014exact}, and gradient descent in factorized models exhibits implicit low-complexity bias \cite{gunasekar2017implicitregularizationmatrixfactorization}. Our starting point is the singular-value dynamics derived by \citet{arora2019implicitregularizationdeepmatrix}, which characterizes the evolution of ordered singular values of the end-to-end linear map under gradient descent in the balanced regime. This result gives a clean spectral-level view of learning dynamics and serves as a baseline for mode-wise trajectories such as those in Fig. \ref{fig:intro}.
The balancing hypothesis is restrictive: it prevents direct consideration of standard random initialization and is tightly coupled to end-to-end linearity. It is therefore unclear how to use it as a mechanism for Jacobian spectra in more general architectures. In parallel, Jacobian spectra have been extensively studied through signal propagation and dynamical isometry at random initialization \cite{schoenholz2017deep, pennington2017resurrecting, pennington2018emergence}, and through empirical sensitivity measures during training \cite{novak2018sensitivity}. These lines of work illuminate trainability regimes and bulk or extremal spectral behavior, but they do not directly yield predictions for the evolution of ordered singular values at finite depth in a way that explains strong separation, nor do they provide a structural replacement for balancing that would enable deep-linear-like spectral dynamics.

%Our goal is to recover deep-linear-like singular-value dynamics without assuming balancing, under assumptions that are compatible with random initialization and better aligned with the Jacobian structure induced by gating. The central claim of this paper is a mechanism: depth-driven scaling produces spectral separation, and spectral separation stabilizes singular directions through alignment. Once dominant directions stabilize, singular-value evolution becomes approximately decoupled, yielding deep-linear-like dynamics without balancing.

Our goal is to recover deep-linear-like singular-value dynamics without assuming balancing, under conditions compatible with random initialization and with the gated-product structure of Jacobians. We argue that a simple depth-driven mechanism can substitute for balancing: (i) depth-induced scaling creates separation among ordered singular values, (ii) sufficiently strong separation promotes alignment of dominant singular vectors across products, and (iii) once leading directions stabilize, singular-value evolution becomes approximately decoupled. This yields deep-linear-like dynamics without balancing.
%
%We make this mechanism concrete in gated products motivated by Jacobians of piecewise-linear networks. For a fixed activation pattern, the Jacobian is exactly a product of weight matrices interleaved with diagonal gates, which is the structure we isolate. We introduce fixed-gates and masked linear networks as controlled models that capture this gated-product form. 
%
We make this mechanism concrete using gated matrix products motivated by Jacobians of piecewise-linear networks. Conditioning on a fixed activation pattern, the Jacobian reduces exactly to a product of weight matrices interleaved with diagonal gates; we isolate this structure via fixed-gates and masked linear networks. Concretely, we (i) establish a depth-scaling characterization for masked products at random initialization, including explicit Lyapunov exponents in a tractable masked model and finite-depth corrections, yielding a principled source of spectral separation; (ii) show that strong spectral separation is sufficient to force alignment of dominant singular vectors in matrix products, implying an approximately shared singular basis across intermediate Jacobians; and (iii) combine these ingredients to motivate an approximation regime where singular-value evolution becomes effectively decoupled, mirroring the functional form of balanced deep-linear dynamics up to a global scaling. Our analysis justifies treating depth scaling and spectral separation as plausible structural conditions, and they guide the approximations that yield a deep-linear-like singular-value dynamic for fixed-gates networks without assuming balancing.

% =========================================================
% ===================== related work ======================
% =========================================================

\section{Related Works}

\textbf{Jacobian spectra, signal propagation, and training-time geometry}

The spectral properties of the input--output Jacobian have been extensively studied, especially at random initialization. Signal propagation analyses characterize regimes of stability and chaos through Jacobian norms and singular value distributions in the infinite-width limit, leading to the notions of edge of chaos and dynamical isometry \cite{schoenholz2017deep, pennington2017resurrecting, pennington2018emergence}. These works primarily describe bulk statistics or extremal behavior via mean-field approximations, and they do not track ordered singular values individually or their finite-depth separation. Complementary empirical studies relate Jacobian sensitivity to generalization during training \cite{novak2018sensitivity}, and \cite{oymak2019generalization} reports that after training, MLPs and CNNs often exhibit a spectrum with a few large singular values, linking this spectral separation to generalization.
%
%In contrast, we study ordered singular values at finite depth in gated product models and formalize a depth-driven route to spectral separation that does not rely on infinite-width limits or isometry assumptions. Our depth-scaling result is proved at random initialization in a masked setting and used to guide approximations beyond that regime.
In contrast, we study ordered singular values at finite depth in gated product models and formalize sufficient conditions for spectral separation that do not rely on infinite-width limits or isometry assumptions.

\textbf{Products of random matrices and Lyapunov exponents}
The asymptotic behavior of singular values of products of random matrices is classically described by the multiplicative ergodic theorem \cite{FurstenbergKesten1960,Oseledets1968}, which guarantees the existence of Lyapunov exponents under suitable conditions, typically including invertibility. Precise characterizations of Lyapunov spectra for Gaussian products are known in several settings \cite{Newman1986, CrisantiPaladinVulpiani1993}. However, these results generally do not apply directly to products involving rank-deficient or structured factors.
Neural Jacobians often take the form of products of matrices interleaved with diagonal gating operators, and gating can make these factors singular, placing them outside standard invertible-product theory. We extend Lyapunov-type reasoning to this gated setting by introducing conditioned $(r,p)$-gates, allowing us to characterize depth-induced exponential scaling and the resulting spectral separation for masked linear networks at realistic depths.

\textbf{Implicit bias in deep linear and factorized models}
For deep linear networks, the implicit bias of gradient descent is well understood. Exact singular-value dynamics were derived by \citet{saxe2014exact}, revealing depth-dependent amplification and low-rank bias. Subsequent work generalized these insights to matrix and tensor factorizations, showing that gradient descent implicitly favors low-complexity solutions that can be related to nuclear norms or analogous measures \cite{gunasekar2017implicitregularizationmatrixfactorization, arora2019implicitregularizationdeepmatrix}. These analyses rely crucially on linearity and on assumptions such as balancing that enforce alignment across layers.
We recover a deep-linear-like mode-wise dynamics in the presence of fixed gating, but without assuming balanced initialization. Related recent evidence suggests that low-complexity effects induced by factorization may also matter in modern architectures, for example through connections between nuclear-norm-like biases and out-of-context generalization behavior in transformer settings \cite{huang2025generalizationhallucinationunderstandingoutofcontext}.

\textbf{Gated and nonlinear architectures}
Recent works study implicit bias in models with gating or structured nonlinearities. \citet{lippl2022implicit} analyze generalized gated linear networks and characterize asymptotic solutions selected by gradient descent under shared-parameter constraints. \citet{jacot2022implicit} introduces a function-space notion of rank and shows that infinitely deep nonlinear networks exhibit a bias toward low-rank functions. While these works provide asymptotic or equilibrium characterizations, they do not describe finite-depth, ordered spectral dynamics of Jacobians induced by fixed activation patterns.
Our focus is instead on finite-depth behavior and on a mechanism in which spectral separation alone induces alignment of singular vectors in products, independently of balancing or infinite-depth limits.

\textbf{Alignment and spectral gaps}
Alignment phenomena in deep linear networks have been studied by \citet{ji2019gradient, ji2020directional}, who show that gradient descent aligns singular directions across layers under suitable conditions. Classical matrix perturbation theory provides tools for subspace stability under spectral gaps \cite{Davis1970, StewartSun1990}. We complement these perspectives by showing that strong spectral separation itself forces singular-vector alignment in matrix products, yielding quantitative convergence rates and providing the structural ingredient that underlies our approximate decoupled singular-value dynamics.

\textbf{Low-rank bias beyond weights}
Several works study rank minimization of network weights or explicit regularization of Jacobians \cite{timor2022implicitregularizationrankminimization, galanti2024sgdweightdecaysecretly, scarvelis2024nuclearnormregularizationdeep}. These approaches promote low rank through architectural or algorithmic choices. Our contribution is orthogonal: low-rank structure and spectral separation can emerge implicitly in the input--output Jacobian due to depth-induced effects in gated products, even without explicit regularization.

% =========================================================
% ===================== Overview and notations ======================
% =========================================================

\section{Overview and Notations}
Unless specified otherwise, all matrices are square $n \times n$ real matrices. If $(M_i)$ is a sequence of matrices and $a \le b$, let $M_{a:b} \coloneq M_b \dots M_a$. If $a>b$, we set $M_{a:b} \coloneq I$ the identity matrix.
In any singular value decomposition (SVD) $M = USV^{\top}$, the singular values $s_1(M),\dots,s_n(M)$ (i.e. the coefficients of $S$) are ordered from highest to lowest.
Our contribution is built upon different blocks, with dependencies shown below.

\textbf{Setup and baseline.}
Section~\ref{sec:motivation} introduces fixed-gates and masked linear networks as models for Jacobian products, and recalls the balanced deep-linear baseline that motivates our target dynamics (Proposition~\ref{prop:gated_linear_network_dot_sk}).

\textbf{Depth scaling at initialization.}
Section~\ref{sec:rapid-decay-init} studies masked products at random initialization and proves depth scaling, including finite-depth corrections (Theorem~\ref{th:singular value asymptotics}). This block is self-contained and does not assume balancing.

\textbf{Alignment from spectral separation.}
Section~\ref{sec:alignment}  isolates a matrix-product phenomenon: under spectral separation, dominant singular directions align across products (Theorem~\ref{th:alignment}). This result is asymptotic in depth and is used as a structural ingredient in the synthesis.

\textbf{Approximate singular-value dynamics.}
Section~\ref{sec:applications} combines the two structural ingredients, depth scaling and spectral separation-induced alignment, as assumptions to obtain a deep-linear-like singular-value evolution for fixed-gates networks up to a multiplicative factor (Proposition~\ref{prop:singular_propto_fixed_gates}; see also Remark~\ref{rem:singular propto}). Theorems~\ref{th:singular value asymptotics} and \ref{th:alignment} are not chained into a fully rigorous training-time statement; rather, they motivate the approximation regime used in this section.

\textbf{Experiments.} Section \ref{sec:experiments} collects plots of experiments illustrating the validity of our results and of the subsequent assumptions they motivate.

Appendices follow the same order: proofs of the baseline dynamics (Appendix~\ref{app:gated_linear_network_dot_sk}), background on exterior powers (Appendix~\ref{app:Recollections on exterior powers}), proofs of the main theorems and propositions (Appendices~\ref{app:proof of spectrum at init}--\ref{app:proofs of sing val dynamics}), definition of diagonal correlation used in section \ref{sec:experiments} (Appendix \ref{app:diagonal correlation}) and additional experiments (Appendix \ref{app:additional_xps}).

% =========================================================
% ===================== Motivation ======================
% =========================================================

\section{Fixed-Gates Products as Jacobian Models}
\label{sec:motivation}

We introduce fixed-gates linear networks as controlled models for Jacobian products along fixed activation patterns in piecewise-linear networks. We then recall a balanced singular-value dynamics result for masked networks that motivates our target approximation.

\begin{definition}[Fixed-Gates Linear Network]
\label{def:fgln}
Let $L \ge 1$. Let $W_\ell \in \mathbb{R}^{n_\ell \times n_{\ell-1}}$ for $\ell=1,\dots,L$ be trainable weight matrices, and let
$D_\ell \in \mathbb{R}^{n_\ell \times n_\ell}$ for $\ell=1,\dots,L-1$ be fixed diagonal matrices (the \emph{gates}).
A \emph{Fixed-Gates Linear Network (FGLN)} of depth $L$ is the linear map
\begin{equation*}
J \coloneq W_L D_{L-1} W_{L-1}\cdots D_1 W_1 .
\end{equation*}
For convenience, we set $D_0 \coloneq I_{n_0}$ and $D_L \coloneq I_{n_L}$ when needed.
In the special case where each $D_\ell$ has diagonal entries in $\{0,1\}$, we call it a \emph{Masked Linear Network}. In the special case where each $W_{\l}$ is square of size $n$, we call it a \emph{square FGLN of width $n$}.
\end{definition}

\begin{remark}
\label{rem:rectangular}
We allow rectangular layers. When we refer to singular values of partial products, they are taken with the usual SVD convention for rectangular matrices.
\end{remark}

\begin{definition}[Multi-mode FGLN]
\label{def:mm-fgln}
Let $\mathcal{M}$ a set of modes. For each mode $m \in \mathcal{M}$, let
\[
J^{(m)} \coloneq W_L D^{(m)}_{L-1} W_{L-1}\cdots D^{(m)}_1 W_1
\]
be an FGLN that shares the same weights $(W_\ell)_{\ell=1}^L$ but may have mode-dependent fixed gates $(D^{(m)}_\ell)_{\ell=1}^{L-1}$.
Let $\sigma$ be a mode function mapping each input $x$ to an index $\sigma(x)\in \mathcal{M}$.
A \emph{Multi-mode FGLN} is the pair $\big(\sigma, (J^{(m)})_{m\in \mathcal{M}}\big)$, and it represents the piecewise-linear map $x \mapsto J^{(\sigma(x))}x$.
\end{definition}

\begin{remark}
An MLP without biases is a Multi-Modes Fixed-Gates Linear Network with a trainable mode matching function $\sigma$.
\end{remark}

% \begin{proof}
% For any input $x$, given an MLP without biases $f$ and its Jacobian $J(x)$, we have the property that $f(x) = J(x)x = J^{(\sigma(x))}x$ for some mode function $\sigma$.
% \end{proof}

% \red{does this hold : on any region of input space where the activation pattern is constant, the gates $D_\ell(x)$ are constant, and the network is linear:
% there exists a matrix $J^{(m)}$ such that for all $x$ in that region, $f(x)=J^{(m)}x$ and $J(x)=J^{(m)}$ almost everywhere.
% In particular, for ReLU we have $D_\ell(x)$ with diagonal entries in $\{0,1\}$, so the resulting Jacobian is a masked (fixed-gates) product.?? if yes refine the proof is ambiguous }

% \blue{In fact for an MLP with arbitrary activation functions the gates pattern move continuously with the input $x$ and the partition appear only in piece-wise linear activation functions like Relu. So in general the number of modes is infinite. Nevertheless we can still consider each mode individually in the possibly infinite set $\mathcal{M}$. Each mode corresponding to a jacobian of the MLP. So yes above is true but there is possibly only one input point where the linearization of the network for a particular mode holds}

In the case of a Masked Linear Network we can extend the analysis done on deep linear networks dynamics from \cite{arora2019implicitregularizationdeepmatrix}. The proof is found in Appendix \ref{app:gated_linear_network_dot_sk}.

\begin{proposition}[Balanced singular-value dynamics for masked linear networks]
\label{prop:gated_linear_network_dot_sk}
Consider a Masked Linear Network $J(t)=W_L(t) D_{L-1}\cdots D_1 W_1(t)$ trained by gradient flow on the weights $(W_\ell(t))_{\ell=1}^L$ with loss $\loss(J)$.
Let $D_0\coloneq I$ and $D_L\coloneq I$, and define $M_\ell(t)\coloneq D_\ell W_\ell(t) D_{\ell-1}$ for $\ell=1,\dots,L$.
Assume the \emph{balancing hypothesis at initialization}:
\begin{equation*}
M_{\ell+1}^\top(0) M_{\ell+1}(0) \;=\; M_\ell(0) M_\ell^\top(0), \qquad \ell=1,\dots,L-1 .
\end{equation*}
Under gradient flow, this balancing condition is preserved along training.
Let $J(t)=\sum_k s_k(t) u_k(t) v_k(t)^\top$ be an SVD of $J(t)$.
Then the singular values satisfy, for each $k$,
\begin{equation*}
\dot{s}_k(t) = -L\, s_k(t)^{2-2/L}\, \left\langle \nabla_J \loss(J(t)),\, u_k(t) v_k(t)^\top \right\rangle,
\label{eq:svd_ode}
\end{equation*}
where $\langle A,B\rangle=\mathrm{Tr}(A^\top B)$ is the Frobenius inner product.
\end{proposition}

Proposition \ref{prop:gated_linear_network_dot_sk} relies crucially on linearity and on a balancing invariant that is preserved by gradient flow. In the rest of the paper we take a different route by isolating two structural properties of deep gated products that can be stated and tested directly at the level of Jacobians.

\begin{definition}[Depth scaling]
\label{def:depth-scaling}
Let $M_{1:\ell}\coloneq M_\ell \cdots M_1$ denote partial products.
We say that a family of products satisfies \emph{$\varepsilon$-depth scaling} at rank $r$ if there exist constants $(\gamma_k,\delta_k)_{k=1}^r$ such that for all $\ell\in\{1,\dots,L\}$ and all $k\le r$,
\begin{equation*}
\left|\log s_k(M_{1:\ell}) - (\ell \gamma_k + \delta_k)\right| \le \varepsilon .
\end{equation*}
\end{definition}

\begin{definition}[Spectral separation (working definition)]
    \label{def:spectral-separation}
    Let $\varepsilon >0$. We say that a matrix $M$ exhibits \emph{$\varepsilon$-spectral separation} at rank $r$ if for all $k \le r$
    \begin{equation*}
        \frac{s_{k+1}(M)}{s_k(M)} < \varepsilon
    \end{equation*}
    with the convention $s_{n+1}(M)=0$.
\end{definition}

In the next sections we study how depth scaling arises at random initialization in Fixed-Gates Linear Networks and thus MLP Jacobians, how spectral separation induces alignment of singular directions in products, and how these ingredients lead to an approximate deep-linear-like singular-value evolution for Fixed Gates Linear Networks.

% =========================================================
% ===================== Depth Scaling ======================
% =========================================================

\section{Spectral separation and depth scaling at Initialization}
\label{sec:rapid-decay-init}

In this section, $n \ge 1$ denotes a fixed integer. We consider $n \times n$ real matrices unless specified otherwise.

Consider the gated product $J_L = W_L D_{L-1} W_{L-1}\cdots D_1 W_1$ at initialization. Denote $s_{i,L}$ its $i$-th singular value.
Experimentally, $\frac{1}{L}\log s_{i,L}$ converges slowly to some value $\gamma_i$, called the \emph{Lyapunov exponent}.

When all gates $D_{\l}$ are set as the identity matrix the existence of the exponents $\gamma_i$ is a known fact from random matrix theory, see for instance \cite{bougerol2012products}. More generally if $J_L$ is a product of random invertible i.i.d. matrices satisfying mild assumptions, one is able to show the existence of Lyapunov exponents. In general, the product $D_\l W_\l$ may not be invertible, preventing us from using this fact as a black box.

% {\color{magenta}It seems the key result here is~\cite{ledrappier2006quelques} (see Proposition 1.1.).}. 
% \brown{[Cannot find the statement in this reference]}

We prove the existence of Lyapunov exponents $\gamma_i$ for $J_L$ at initialization, and we compute them explicitly in a Bernoulli-gated Gaussian model. To account for finite-depth effects, we also derive a first-order correction in expectation:
\[
\mathbb E\!\left[\frac{1}{L}\log s_{i,L}\right] = \gamma_i + \frac{d_{i-1}-d_i}{L} + o\!\left(\frac{1}{L}\right),
\]
where the constants $d_i$ are universal in the sense that they depend only on $i$ (through a Haar-orthogonal minor) and not on $(p,\sigma)$.
This correction yields a more accurate approximation for the top of the spectrum at moderate depth.

% {\color{magenta}You cite the Theorem before stating it. I recommend to add a/the statement (at least an informal statement) if you make a reference to this theorem. If you don't want to add an informal statement here, maybe move the paragragp below Theorem 5.3.?}

\subsection{Terminology and statement}

\begin{definition}
    Let $0 < p < 1$, and $\sigma > 0$.
    \begin{itemize}
        \item We say a random $n \times n$ matrix $W$ is \emph{$\sigma$-Ginibre} if it has i.i.d.entries distributed as $\mathcal{N}(0,\sigma^2)$.
        \item We say a random $n \times n$ matrix $D$ is a \emph{$p$-gate} if it is diagonal with i.i.d. Bernoulli$(p)$ diagonal entries.
    \end{itemize}
\end{definition}

Consider the product $J_L = M_L\cdots M_1$ where each layer is $M_\l=D_\l W_\l$,
with $D_\l$ a $p$-gate and $W_\l$ a $\sigma$-Ginibre matrix.

When analyzing $J_L$ as $L\to\infty$ a technical degeneracy arises: almost surely, some gate becomes the zero matrix at a sufficiently large depth, forcing $J_L$ to be eventually zero. This phenomenon is irrelevant at the moderate depths and widths used in practice, where gates have rank concentrated near $np$. To obtain a mathematically meaningful infinite-depth limit while staying faithful to the practical regime, we condition gates to have rank at least $r$, leading to the $(r,p)$-gate model below.

\medskip

\begin{definition}\label{def:(r,p)-gates and layers}
        Let $1 \le r \le n$ be an integer, $0 < p < 1$, and $\sigma > 0$.
    \begin{itemize}
        \item We say a random $n \times n$ matrix $D$ is an \emph{$(r,p)$-gate} if it is distributed as a $p$-gate conditioned to having rank at least $r$.
        \item We say a random $n \times n$ matrix $M$ is an \emph{$(r,p,\sigma)$-layer} if it is of the form $DW$, with $D$ an $(r,p)$-gate, $W$ a $\sigma$-Ginibre matrix, $D$ and $W$ independent.
    \end{itemize}
\end{definition}
 Each diagonal entry of an $(r,p)$-gate $D$ follows a Bernoulli$(p)$ distribution, but these entries are \emph{not} independent. However in practice $D$ almost behaves as a $p$-gate. Indeed, to sample an $(r,p)$-gate one simply has to sample successive $p$-gates and keep the first one with rank $\ge r$. For $r \ll np$, this holds immediately with very high probability.

 We state our main theorem.

\begin{theorem}\label{th:singular value asymptotics}
    Let $1 \le r \le n$ be an integer, $0 < p < 1$, and $\sigma > 0$. Consider a sequence of i.i.d. $(r,p,\sigma)$-layers $(M_i)$, and write $M_i = D_iW_i$. For any positive integer $L$, define the random matrix

    $$J_L = M_L \cdots M_1 = (D_L W_L)\cdots (D_1 W_1)$$

    Denote $s_{1,L} \ge \dots \ge s_{n,L}$ the singular values of $J_L$. Then there exists explicit real numbers called \emph{Lyapunov exponents}
    $$\gamma_1 > \dots > \gamma_r$$
    which depend on $r,p$ and $\sigma$ such that almost surely
    $$\frac{1}{L} \log(s_{i,L}) \xrightarrow{L \to \infty}
    \begin{cases}
        \gamma_i & \text{if } 1 \le i \le r \\
        -\infty & \text{if } i > r.
    \end{cases}$$
    Moreover, for $1 \le i \le r$, the expectation satisfies
    $$\mathbb E \left[ \frac{1}{L} \log(s_{i,L}) \right] = \gamma_i + \frac{d_{i-1}-d_i}{L} + o\left( \frac{1}{L}\right)$$
    with
    $$d_i = -\mathbb{E}[\log |\det \Omega^{i,i}|]$$
    where $\Omega^{i,i}$ is the $i \times i$ upper left block of a Haar-distributed orthogonal matrix (and $d_0 = 0$).
\end{theorem}

A proof of Theorem \ref{th:singular value asymptotics} is provided in Appendix \ref{app:proof of spectrum at init}. Our argument is general; although the Theorem is stated for Masked Linear Networks, the proof applies to many other choices of initialization for FGLN (Theorem \ref{th:exp spectrum of iterated product}).

\begin{remark}[Closed form for $(r,p,\sigma)$-layers with Ginibre weights]
\label{rem:closed-form-gamma}
Under the assumptions of Theorem~\ref{th:singular value asymptotics}, for each $1\le i\le r$ the Lyapunov exponent is
\begin{equation*}
\gamma_i
= \ln(\sqrt{2}\sigma)
+ \frac{\sum_{t=r}^n \left[ \binom{n}{t}p^t(1-p)^{n-t} \psi\!\left( \frac{t-i+1}{2} \right)\right]}
{2\sum_{m=r}^n \binom{n}{m} p^m (1-p)^{n-m}},
\end{equation*}
where $\psi$ is the digamma function.
\end{remark}

\begin{remark}[High-probability regime $r \ll np$]\label{rem:high-probability-regime}
When $r$ is substantially smaller than $np$, the conditioning event $\{\mathrm{rank}(D)\ge r\}$ has probability close to $1$, so $(r,p)$-gates behave similarly to i.i.d.\ $p$-gates.
In this regime, the normalizing denominator in the expression for $\gamma_i$ is close to $1$, yielding the approximation independent of $r$
\[
\gamma_i \approx \ln(\sqrt{2}\sigma) + \frac{1}{2}\sum_{t=r}^n \binom{n}{t}p^t(1-p)^{n-t}\,
\psi\!\left(\frac{t-i+1}{2}\right),
\]
which is the expression used in our numerical comparisons.
\end{remark}

 From $\gamma_i > \gamma_{i+1}$ we find the following corollary.

\begin{corollary}\label{cor:exp spectrum for layers}
    Let $1 \le r \le n$ be an integer, $0 < p < 1$, and $\sigma > 0$. Consider a sequence of i.i.d. $(r,p,\sigma)$-layers $(M_i)$, and write $M_i = D_iW_i$. For any positive integer $L$, define the random matrix
    $$J_L = M_L \dots M_1 =  D_L W_LD_{L-1} \dots W_2D_1W_1.$$
    Denote $s_{1,L} \ge \dots \ge s_{n,L}$ the singular values of $J_L$. Then for any $1 \le i < r$, almost surely $s_{i,L} \neq 0$ and
    $$\frac{s_{i+1,L}}{s_{i,L}} \xrightarrow{L \to \infty} 0.$$
\end{corollary}

Our results indicate that Jacobians of FGLNs at initialization exhibit both spectral separation and depth scaling. This is confirmed by our experiments in subsection \ref{subsec:depth scaling}.

\section{Alignment of singular vectors through spectral separation}\label{sec:alignment}

In this section, $n \ge 1$ denotes a fixed integer. We consider $n \times n$ real matrices unless specified otherwise. The goal is to formalize the following phenomenon. If a matrix $A$ has strong spectral separation, then its dominant left singular directions are stable under right-multiplication: for a wide class of matrices $B$, the product $AB$ has left singular vectors that are close to those of $A$ on the top singular subspace. Geometrically, $A$ maps the unit sphere to a highly elongated ellipsoid, and such an extreme anisotropy suppresses the influence of subsequent moderate distortions induced by $B$.
We state a quantitative alignment theorem with explicit first order convergence rates (Appendix~\ref{app:proof of alignment} and Remark \ref{rem:first order alignment}).

Consider a product $J_L=M_L\cdots M_1$ as in fixed-gates Jacobians, and fix $\ell \ll L$. Write
\[
A_\ell \coloneq M_L\cdots M_{\ell+1}, \qquad B_\ell \coloneq M_\ell\cdots M_1,
\]
so that $J_L = A_{\l}B_{\l}$.

In the i.i.d.\ initialization setting, each suffix product $A_\ell$ has the same distribution as a shorter product of $(r,p,\sigma)$-layers, and thus inherits spectral separation properties analogous to Corollary~\ref{cor:exp spectrum for layers}. Theorem~\ref{th:alignment} below then predicts that the dominant left singular subspace of $J_L$ is close to that of $A_\ell$, allowing us to relate singular directions of intermediate Jacobian products to those of the full Jacobian. This is confirmed by our experiments in subsection \ref{subsec:spectral separation-induced alignment}.

\begin{theorem}
\label{th:alignment}
Let $(A_L)$ be a sequence of invertible matrices such that each $A_L$ admits an SVD of the form
\[
A_L = U_{A_L} S_{A_L} I^\top,
\]
that is, the right singular vectors are the standard basis. Let $B$ be a fixed matrix and set $J_L \coloneq A_L B$, with SVD
\[
J_L = U_{J_L} S_{J_L} V_{J_L}^\top,
\]
chosen so that the $i$-th columns of $U_{A_L}$ and $U_{J_L}$ have positive dot product for $i\le r$.
Assume there exists an integer $1\le r\le n-1$ such that:
\begin{enumerate}
\item[(i)] (Spectral separation) Let $s_{i,A_{L}}$ be the diagonal values of $S_{A_L}$. Then
\[
\forall\,1\le i\le r,\qquad \frac{s_{i+1,A_L}}{s_{i,A_L}} \xrightarrow[L\to\infty]{} 0 .
\]
\item[(ii)] The first $r$ columns of $B$ are linearly independent.
\end{enumerate}
Let $R_L \coloneq U_{A_L}^\top U_{J_L}$. Write $X^{r,r}$ for the upper-left $r\times r$ block of a matrix $X$. Then
\[R_L^{r,r}\to I_r\]
and the off-diagonal entries converge to $0$ at rates controlled by the spectral ratios in (i). In particular, the (left) singular vectors of $A_{L}$ and $J_L$ align as $L \to \infty$.
\end{theorem}

\begin{remark}[Clustered spectra]
When applying Theorem~\ref{th:alignment} to random products, it may occur that several consecutive singular values of $A_L$ are close, weakening spectral separation within that cluster. A natural extension is to assume separation only between clusters. One then expects $R_L^{r,r}$ to converge to a block-diagonal matrix with orthogonal blocks matching the cluster sizes.
\end{remark}

% {\color{magenta} Again, I understand, this is your work and you want to show you have made some calculations, but I would strongly recommend removing the below as it removes the attention of the reader from the most important equation which is the alignment of the singular vectors. In my opinion the theorem would be way neater if you stop after the $L\rightarrow \infty$ above}
% {\blue{This part is actually quite crucial. This is the statement that says that there is a geometric decay of the coefficient as we move away from the diagonal. They give the first order term, we believe it is relevant precisely because we experimentally do not find perfectly diagonal matrices, we see the firsy order. Moreover, we use it in Remark \ref{rem:singular propto}. We believe this is a strong part of our argument: controlling first order provides the opportunity for simple yet more accurate analyses.}}

% {\color{magenta}I understand. But do you really think it should be included in the main statement or is it just a detailed version of the statement? You repeat yourself here. If you look above, you already give a short and clear explanation for the decay. Honestly I think it is enough. If the readers want details, they will just look up the proofs. But again that is my opinion. This is your work :)}

\begin{remark}[First order corrections]\label{rem:first order alignment}
    We are able to provide the first order term in the convergence $R_L^{r,r} \to I_r$ of Theorem \ref{th:alignment}, which we will use in Remark \ref{rem:singular propto}. Let $C\coloneq BB^\top$. Then, as $L\to\infty$,
    \begin{align*}
    \big(S_{J_L} S_{A_L}^{-1}\big)^{r,r} &\to \Sigma_C,\\
    \big(S_{A_L}^{r,r}\big)^{-1} R_L^{r,r} S_{A_L}^{r,r} &\to T_C,\\
    S_{A_L}^{r,r} R_L^{r,r} \big(S_{A_L}^{r,r}\big)^{-1} &\to (T_C^{-1})^\top,
    \end{align*}
    where $\Sigma_C$ is diagonal and $T_C$ is lower-triangular with unit diagonal, defined by the Cholesky factorization
    \[
    C^{r,r} = T_C \Sigma_C T_C^\top.
    \]
    Unraveling the definitions, this means that the off-diagonal coefficient of $R_L^{r,r}$ in position $(i,j)$ decays as $\frac{s_{i,A_L}}{s_{j,A_L}}$ (resp. $\frac{s_{j,A_L}}{s_{i,A_L}}$) if $j < i$ (resp. $i < j$). The proportionality constants are related to $T_C$.
\end{remark}

%%%%%%%%%%%%%%%%%%%%%%%%%%%%%%%%%%%%%%%%%%%%%%%%%%%%%%%%%%%%%%%%%%%%%%%%%%%%%%%%%%%%%

\section{Applications}\label{sec:applications}

% We hope that Theorems \ref{th:singular value asymptotics} and \ref{th:alignment} provide manageable assumptions for the study of singular values and singular vectors within Jacobians of neural networks, from which further properties could be inferred. 

Theorems~\ref{th:singular value asymptotics} and \ref{th:alignment} suggest two testable signatures of deep gated products: depth-induced scaling of ordered singular values and separation, induced alignment of dominant singular directions. In this section we illustrate how, under a controlled approximation regime combining these signatures, one can recover a deep-linear-like mode-wise singular-value evolution without assuming balancing. Consider an FGLN 
\[J = W_L D_{L-1}W_{L-1} \cdots D_1 W_1\]
trained by gradient flow on the weights $(W_\ell(t))_{\ell=1}^L$ with loss $\loss(J)$. We denote the SVD of $J = U S V^{\top}$, with $s_1 \ge ~\cdots \ge ~s_n$ its singular values. Let
\begin{align*}
    A_{\l} & \coloneq W_L D_{L-1} W_{L-1} \cdots W_{\l+1}D_{\l}, \\
    B_{\l} & \coloneq D_{\l-1}W_{\l-1} \cdots D_1 W_1
\end{align*}
so that for all $\l$ we have $J = A_{\l} W_{\l} B_{\l}$. Consider their respective SVDs
\[A_{\l} = U_{A_{\l}} S_{A_{\l}} V_{A_{\l}}^\top, \qquad
    B_{\l} = U_{B_{\l}} S_{B_{\l}} V_{B_{\l}}^\top.\]

A derivation in Appendix~\ref{app:proofs of sing val dynamics} yields the mode-wise dynamics summarized below.
\begin{proposition} \label{prop:singular_propto_fixed_gates}
    With notations as above, we make the following assumptions:
    \begin{enumerate}
        \item[(i)] \emph{(Depth Scaling)} For each $1 \le k \le n$ there exists scalar functions $\gamma_k(t)$ and $\delta_k(t)$ such that for every $\l \in \{1,\dots,L\}$ the singular values satisfy
        \begin{align*}
            s_{k,A_{\l}}(t) & = e^{(L-\l+1)\gamma_k(t) + \delta_k(t)}, \\
            s_{k,B_{\l}}(t) & = e^{\l \gamma_k(t) + \delta_k(t)};
        \end{align*}
        \item[(ii)] \emph{(Spectral separation)} There exists some $\varepsilon \ll 1$ such that for all time $t$ and every $\l \in \{1,\dots,L\}$
        \[\|U^{\top} U_{A_{\l}} - I\|_{\infty} < \varepsilon,\]
        \[\|V^{\top} V_{B_{\l}} - I\|_{\infty} < \varepsilon.\]
    \end{enumerate}
    Then for each $1 \le k \le n$
    \[\dot s_k(t) \overset{\varepsilon \to 0}{\sim} -e^{(2+\frac{2}{L}) \delta_k(t)} L s_k(t)^{2 - \frac{2}{L}} \inner{\nabla_{J} \loss(t)}{u_k(t) v_k(t)^\top}.\]
\end{proposition}

%Assumption \emph{(i)} is motivated by Theorem \ref{th:singular value asymptotics}, and assumption \emph{(ii)} is expected to hold in the context of spectral separation due to Theorem \ref{th:alignment}. Ideally, we would formulate assumption \emph{(ii)} exclusively in terms of singular values. However, the quantity $\|U^{\top} U_{A_{\l}} - I\|$ depends not only on the spectrum of $A_{\l}$, but also on the coefficients involved in the Cholesky factorization of $V_{A_{\l}}^{\top} B_{\l-1} B_{\l-1}^{\top} V_{A_{\l}}$. Controlling these coefficients lies beyond the scope of the present work, and we therefore do not pursue such an analysis here.

Assumption (i) is motivated by Theorem~\ref{th:singular value asymptotics} at initialization and serves here as a training-time approximation regime that can be tested empirically. Assumption (ii) formalizes the emergence of an approximately shared singular basis across intermediate products, as suggested by Theorem~\ref{th:alignment} under strong separation. Making (ii) depend only on singular values would require controlling additional quantities (e.g., Cholesky factors arising from $B_\ell B_\ell^\top$ in the alignment expansion), which we leave open.

Up to the multiplicative factor $e^{(2+\frac{2}{L})\delta_k(t)}$, the approximation for $\dot s_k$ coincides exactly with the expression stated in Proposition \ref{prop:gated_linear_network_dot_sk}. In our analysis, the balancing hypothesis has been replaced by an assumption on the exponential spectra of the intermediate Jacobian whose validity during training can in principle be assessed empirically (see for instance subsections \ref{subsec:depth scaling} and \ref{subsec:spectral separation-induced alignment}). As a trade-off we find an asymptotic approximation rather than an exact equality.

%Proposition \ref{prop:singular_propto_fixed_gates} should be thought of as a proof of concept rather than result of interest in itself, especially considering our approximations may be too strong. For instance spectral separation is a property arising in products of sufficiently many matrices; assuming it holds uniformly even for $A_L = W_LD_L$ and $B_1 = D_1W_1$ is unrealistic. Secondly, the effect of spectral separation is to cluster the nonzero coefficients of $U^TU_{A_{\l}}$ around the diagonal, something that the distance to $I$ does not capture faithfully. In the eyes of the authors, having recovered a similar expression to that of Proposition \ref{prop:gated_linear_network_dot_sk} is an indicator of the potential future uses of assumptions of type depth scaling and spectral separation in non-linear settings.

We emphasize that Proposition~\ref{prop:singular_propto_fixed_gates} is an illustrative derivation: its assumptions are deliberately strong and are not expected to hold uniformly across all $\ell$ (e.g., for very short prefix/suffix products). In practice, separation and alignment emerge only once products are sufficiently deep, and alignment is typically approximate and structured (mass concentrates near the diagonal rather than collapsing exactly to $I$). Our experiments therefore evaluate these assumptions quantitatively and identify the regimes where the approximation is accurate.

\begin{remark}\label{rem:singular propto}
    Proposition \ref{prop:singular_propto_fixed_gates} has two weaknesses: the model is restrictive, and its assumptions are crude. In an attempt to mitigate both flaws, we provide a second statement. Consider $M_1(t),\dots,M_L(t)$ any matrix-valued smooth functions, and define
    $J(t) = M_L(t)\dots M_1(t) = M_{1:L}(t)$, 
    which can be considered as an MLP Jacobian, or as an intermediate Jacobian for more sophisticated architecture such as a transformer; we impose no further constraint on the time dynamics. Working in this broader setting, we may replace the order-zero spectral separation approximation in Proposition \ref{prop:singular_propto_fixed_gates} by its first order provided by Remark \ref{rem:first order alignment}. 
    
    The statement is the following. Denote $J = USV$ its SVD, and $s_1\ge \dots \ge s_n$ its singular values. For each $\l \in \{1,\dots,L\}$, define $A_{\l}(t) = M_{\l:L}(t)$,  $B_{\l}(t) = M_{1:\l-1}(t)$ and consider their respective SVDs
    \[A_{\l} = U_{A_{\l}} S_{A_{\l}} V_{A_{\l}}^\top, \qquad
    B_{\l} = U_{B_{\l}} S_{B_{\l}} V_{B_{\l}}^\top.\]
    We make the following assumptions:
    \begin{enumerate}
        \item[(i)] \emph{(Depth Scaling)} For each $1 \le k \le n$ there exists scalar functions $\gamma_k(t)$ and $\delta_k(t)$ such that the singular values satisfy
        \begin{align*}
            s_{A_{\l},k}(t) & = e^{(L-\l+1)\gamma_k(t) + \delta_k(t)}, \\
            s_{B_{\l},k}(t) & = e^{\l \gamma_k(t) + \delta_k(t)};
        \end{align*}
        \item[(ii)] \emph{(spectral separation)} There exists some $\varepsilon \ll 1$, a lower triangular matrix $T_{\l}^-(t)$ and an upper triangular matrix $T_{\l}^+(t)$ such that for all time $t$
        \begin{align*}
            \| S_{A_{\l}}^{-1} U^{\top} U_{A_{\l}} S_{A_{\l}} - T_{\l}^-\|_{\infty} & < \varepsilon, \\
            \| S_{B_{\l}} V_{B_{\l}} V^{\top} S_{B_{\l}}^{-1} - T_{\l}^+\|_{\infty} & < \varepsilon.
        \end{align*}
    \end{enumerate}
    Then for each $1 \le k \le n$
    \[\dot s_k \overset{\varepsilon \to 0}{\sim} e^{(1+\frac{1}{L})\delta_k} s_k^{1-\frac{1}{L}}\sum_{\l = 1}^L (T_{\l+1}^- V_{A_{\l+1}}^{\top} \dot M_{\l} U_{B_{\l-1}} T_{\l-1}^+)_{k,k}.\]
    The reader is once again referred to Appendix \ref{app:proofs of sing val dynamics} for the proof. In this level of generality the expression is primarily diagnostic; for a specific architecture one can expand $\dot M_\ell$ and study which terms dominate.
    %In this level of generality it is difficult to extract further information. When working with a specific architecture, one could expand $\dot M_{\l}$ and carry on the analysis.
\end{remark}

\begin{figure}[t]
    \centering
    \includegraphics[width=0.49\linewidth]{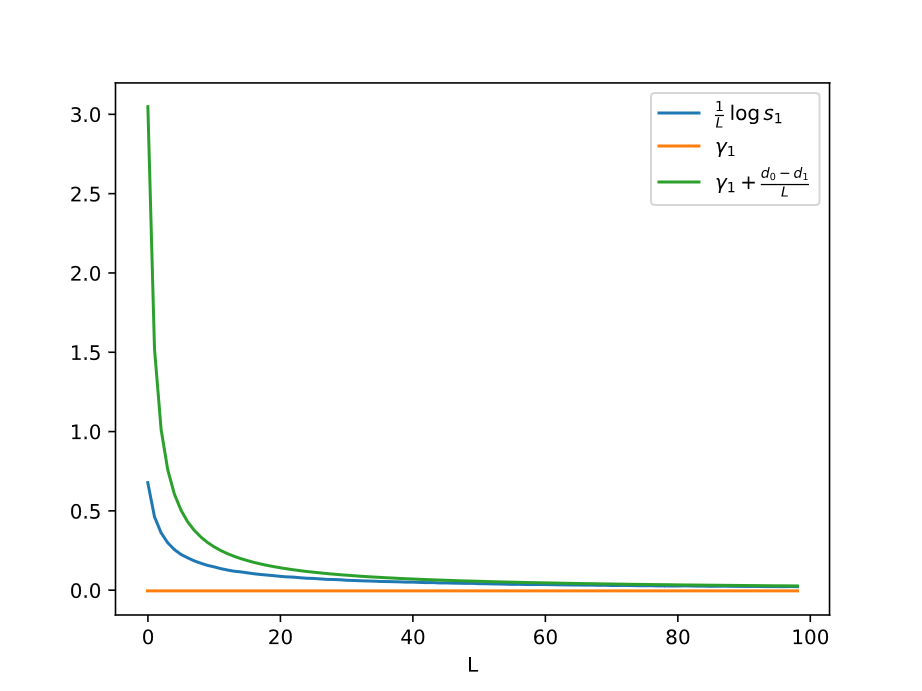}
    \includegraphics[width=0.49\linewidth]{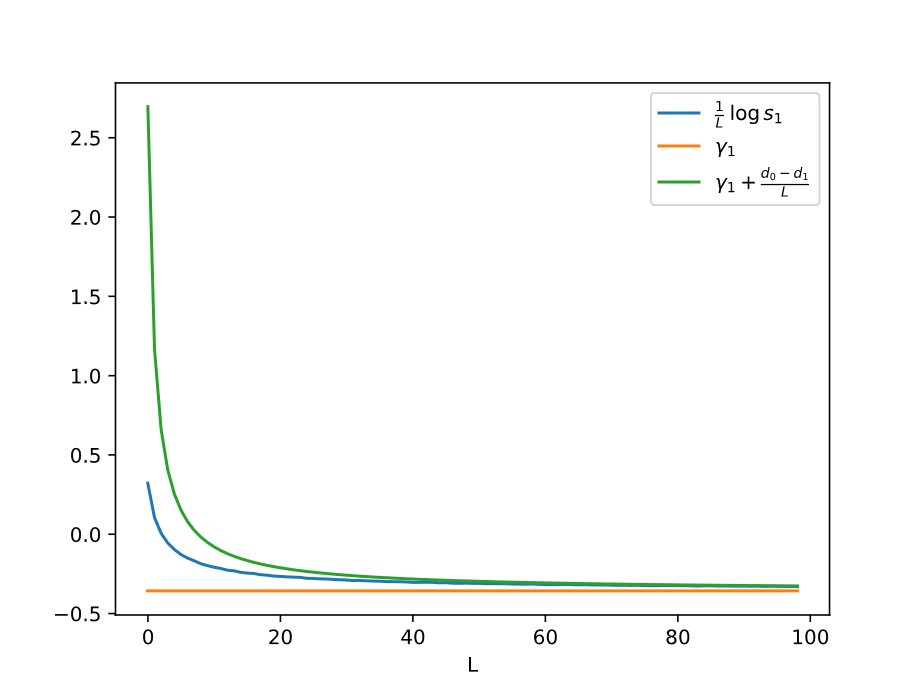}
    \caption{Convergence of $\frac{1}{L}\log s_{1,L}$ to $\gamma_1$ and comparison to the first-order correction $\gamma_1 + \frac{d_0-d_1}{L}$, for Gaussian weights and Bernoulli gates with $p=1$ (left) and $p=0.5$ (right).}
    \label{fig:convergence_s1_gamma1}
\end{figure}
% =========================================================
% ===================== Experiments =======================
% =========================================================

\section{Experiments}\label{sec:experiments}
%We showcase the main experiments that support our findings in this section. We additionally provide more experiments with several parameter variations and other findings in Appendix \ref{app:additional_xps}. 

We empirically evaluate the two signatures underpinning our analysis: (i) depth-induced scaling of ordered singular values at initialization (including the finite-depth correction), and (ii) separation-induced alignment of dominant singular directions in matrix products. Additional experiments and parameter sweeps are reported in Appendix~\ref{app:additional_xps} \footnote{An anonymized implementation and scripts to reproduce all experiments may be accessed at \url{https://anonymous.4open.science/r/ICML26-66F8/}.}.

\paragraph{Experimental setting}
We consider a square FGLN of depth $L$ and width $n=128$ with \emph{unconditioned} $p$-gates (i.i.d.\ Bernoulli$(p)$ on the diagonal) and $\sigma$-Ginibre initialized weights. Here we consider either $p=1$ or $p=0.5$ and always $\sigma = \frac{1}{\sqrt{n}}$. The initialization depth $L$ is  specified for each plot. 

%\red{Recall that Theorem~\ref{th:singular value asymptotics} is stated for conditioned $(r,p)$-gates to avoid a degeneracy as $L\to\infty$. In our finite-depth regime, $\mathrm{rank}(D_\ell)$ concentrates near $np$, so the conditioning event holds with overwhelming probability for the relevant ranks. We therefore use unconditioned $p$-gates in experiments and treat conditioning as a technical devic)e%Appendix~\ref{app:additional_xps} includes a direct sanity check comparing both models.}\blue{We do not have such a comparison. It is not needed: the expression with no $r$ is morally the "true" experimental limit, $r$ refer to the number of relevant singular values to track and depends on the depth of the network considered. The others being at zero anyway.}

We train the FGLN on the MNIST dataset augmented with the auto-augment policy of Cifar10 from torchvision, in which case we fix parameters $L=10$.  Additional first and last layers are added to the FGLN so that the input (size 784) and output (size 10) map seamlessly to the data while still preserving its square structure for the hidden layers. For consistency with the square-matrix theory, we report spectra and alignment statistics computed on the hidden square block (the product of the $n\times n$ layers), excluding the input/output adapters.
We use notations $J_L, U_{J_L}, A_{\l}, U_{A_{\l}}, B_{\l}$ defined as in Section \ref{sec:applications}, as well as $s_{i,L},\gamma_i,d_i$ as in Theorem \ref{th:singular value asymptotics}. We refer to $J_L$ as the full Jacobian, and to $A_{\l},B_{\l}$ as the intermediate Jacobians.

\begin{figure}[t]
    \centering
    \includegraphics[width=0.49\linewidth]{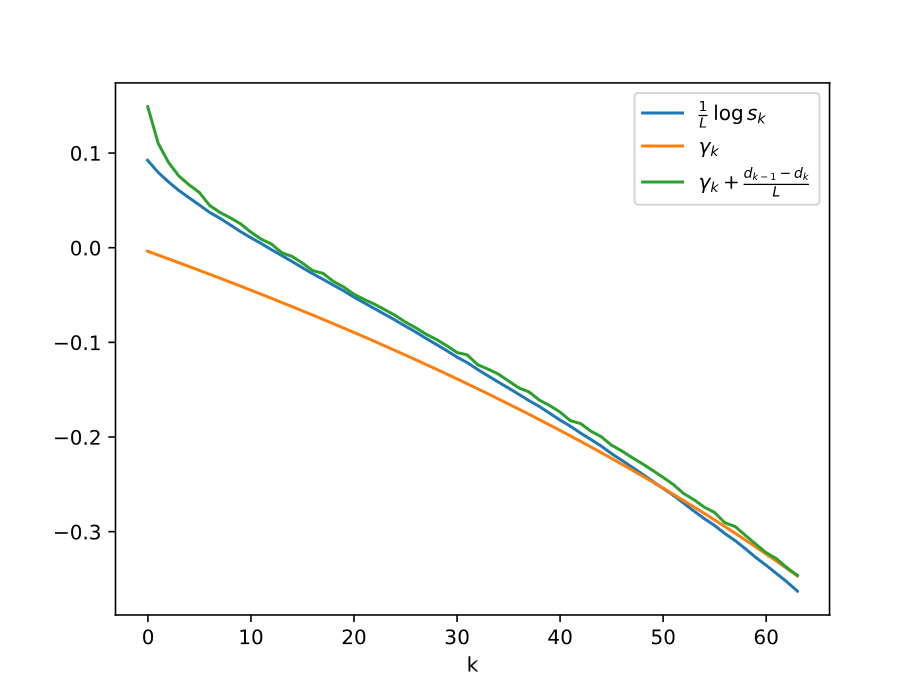}
    \includegraphics[width=0.49\linewidth]{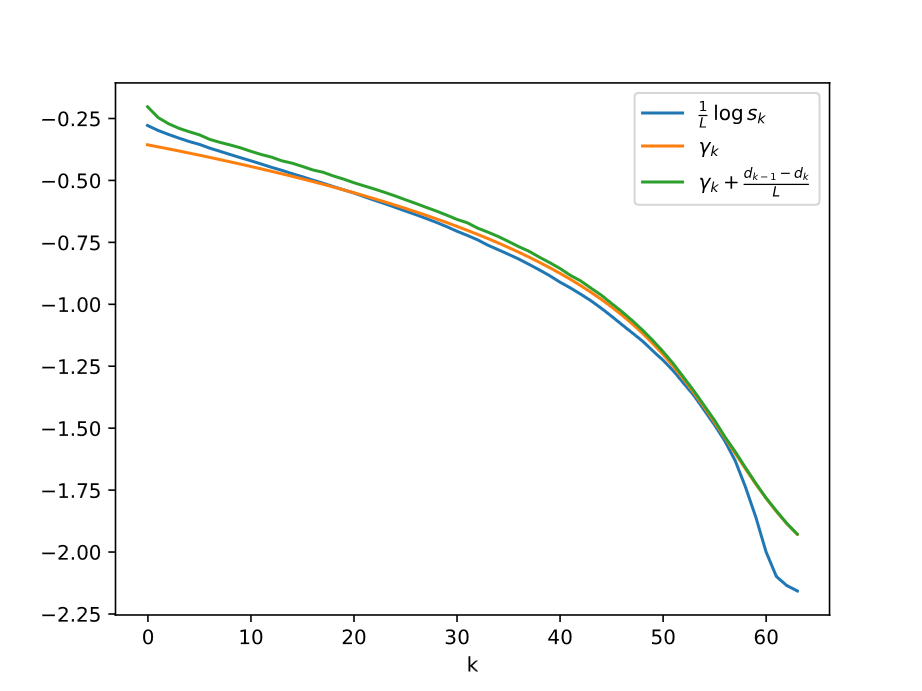}
    \caption{Top 64 values of $\frac{1}{L}\log s_{k,L}$ compared to $\gamma_k$ and to the corrected prediction $\gamma_k + \frac{d_{k-1}-d_k}{L}$, for depth $L=20$ with Gaussian weights and Bernoulli gates with $p=1$ (left) and $p=0.5$ (right).}
    \label{fig:spectrum_sk_gammak}
\end{figure}
\paragraph{Depth scaling}\label{subsec:depth scaling}
We first compare the approximation of $\frac{1}{L} \log(s_{i,L})$ provided by Theorem \ref{th:singular value asymptotics} with its experimental value at initialization (Figure \ref{fig:convergence_s1_gamma1} and Figure \ref{fig:spectrum_sk_gammak}).
% {\color{magenta} Generally, when you say ``Spectrum of $\frac{1}{L}\log s_{k, L}$ or $\log s_{k,\ell}$" in the legend, what you mean is really ``Representation of $\frac{1}{L}\log s_{k, L}$ for the largest $70$ singular values" or ``Representation of the spectrum of $B_{\ell}$"}
%
We plot in Figure \ref{fig:depth_scaling} the depth scaling property of the intermediate Jacobians $B_{\l}$, at initialization and after training. As expected, the log-singular values are affine functions in the depth. 
%The separation in two clusters of singular values after training is a manifestation of the low-rank bias.
After training, the emergence of two clusters is consistent with an effectively low-rank Jacobian geometry.
\begin{figure}[t]
    \centering
    \includegraphics[width=0.49\linewidth]{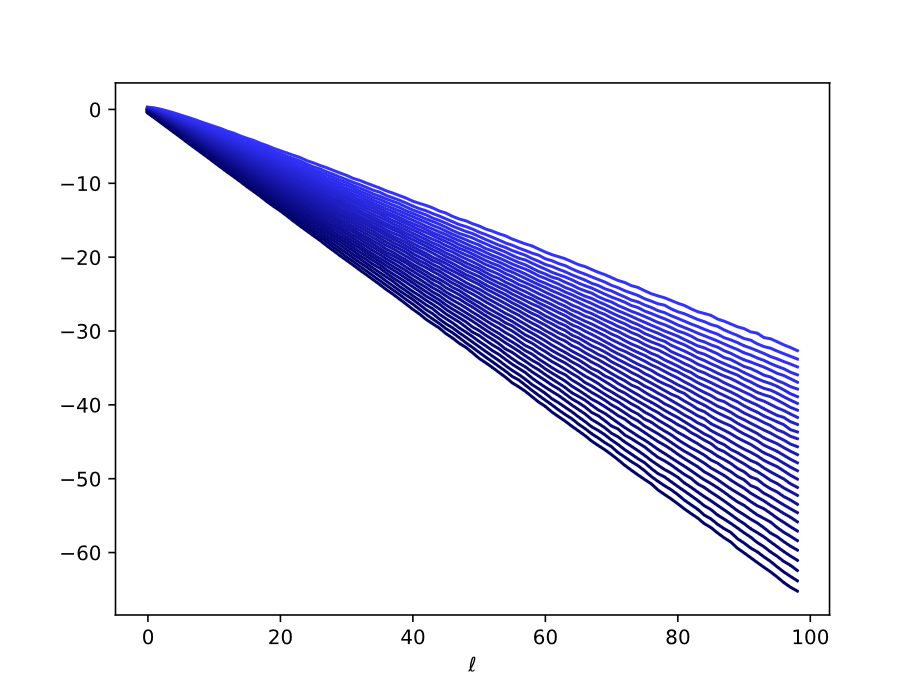}
    \includegraphics[width=0.49\linewidth]{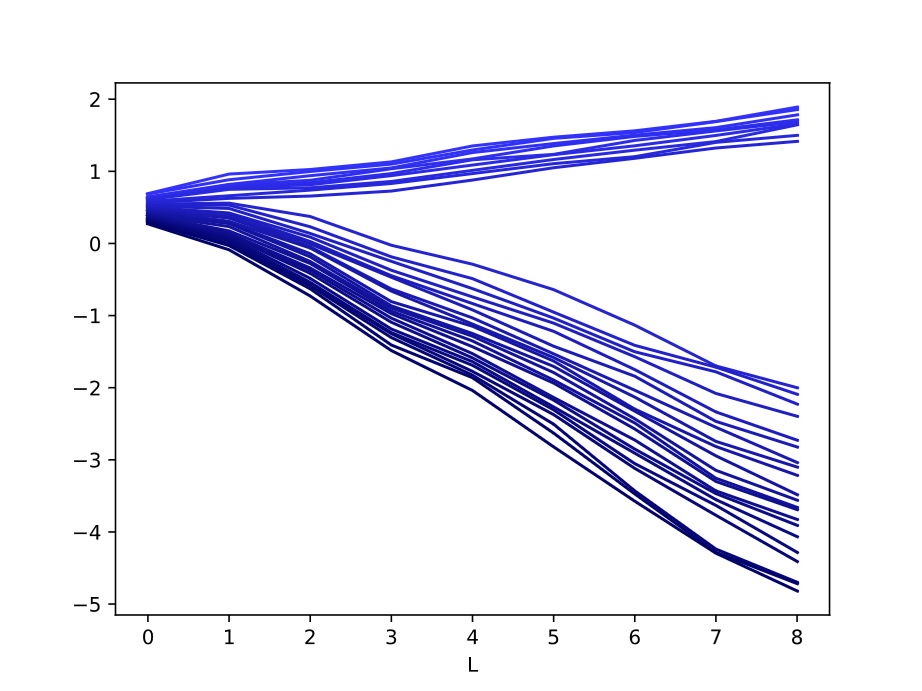}
    \caption{Spectrum of $\log s_{k,\l}$ for the top 30 log-singular values of the intermediate Jacobians $B_{\l}$ as $\l$ varies, at initialization (left, $L=100$) and trained in the MNIST setting (right), $p=0.5$.}
    \label{fig:depth_scaling}
\end{figure}
\begin{figure}[t]
    \centering
    \includegraphics[width=0.49\linewidth]{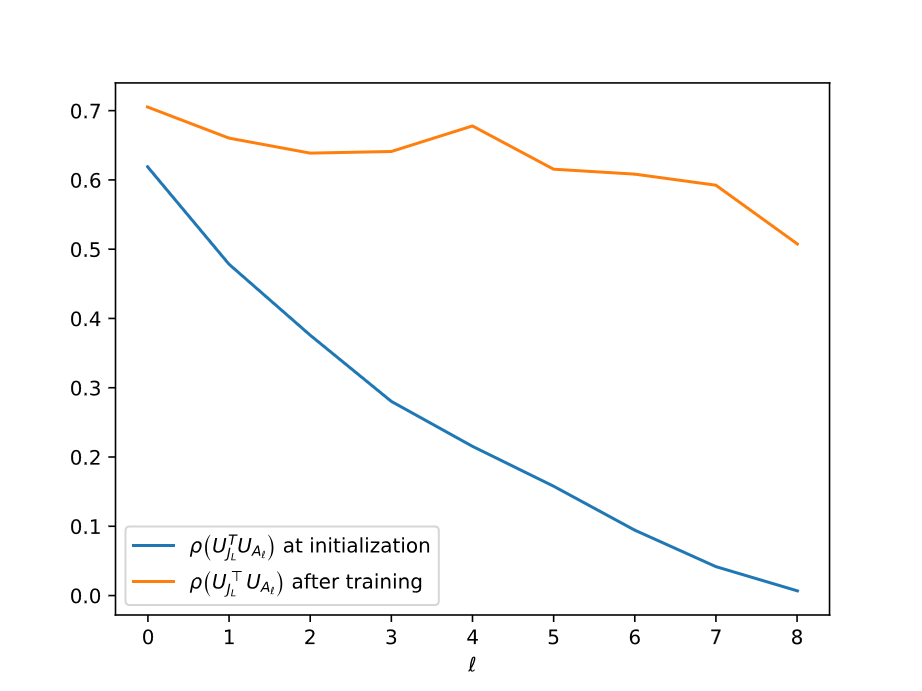}
    \caption{Diagonal correlation coefficient of $U_{J_L}^{\top}U_{A_{\l}}$ at initialization (blue) and trained in the MNIST setting (orange) with $p=0.5$}
    \label{fig:diagonal_correlation_uua}
\end{figure}
\begin{figure}[t]
    \centering
    \includegraphics[width=0.32\linewidth]{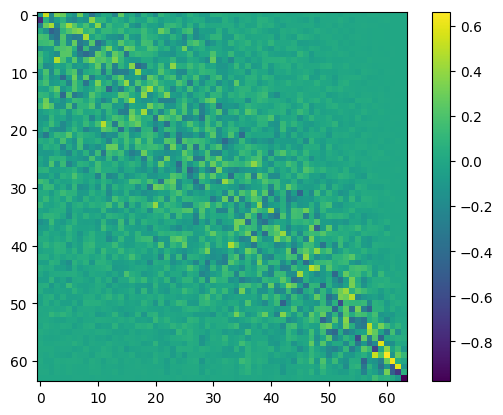}
    \includegraphics[width=0.32\linewidth]{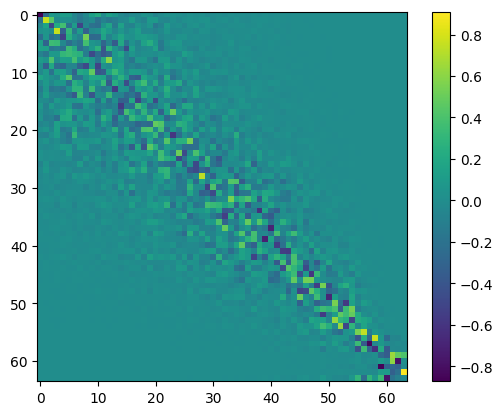}
    \includegraphics[width=0.32\linewidth]{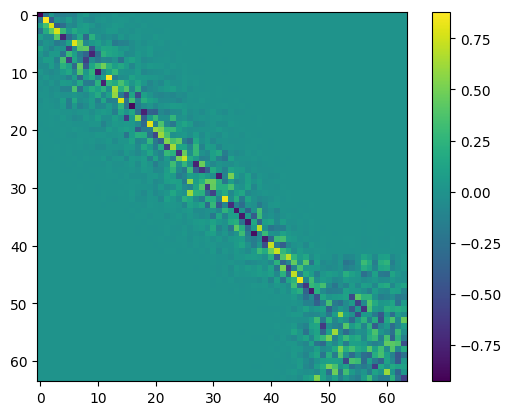}
    \caption{The matrix $U_{J_L}^{\top}U_{A_{\l}}$ at initialization with $p=1$. From left to right: $L=2$, $\l=1$, diagonal correlation of $0.26$; $L=10$, $\l=5$, diagonal correlation of $0.49$ and finally $L=20$, $\l=10$, diagonal correlation of $0.84$.}
\label{fig:alignment_diagonal_examples}
\end{figure}

\paragraph{Spectral separation-induced alignment}\label{subsec:spectral separation-induced alignment}
We show the alignment between the (left) singular vectors of the full Jacobian $J_L$ and the intermediate Jacobian $A_{\l}$. As a metric, we use the diagonal correlation coefficient, which we recall in Appendix \ref{app:diagonal correlation}.
We first compare qualitatively in Figure \ref{fig:alignment_diagonal_examples} the convergence of $U_{J_L}^{\top}U_{A_{\l}}$ to $I$ predicted by Theorem \ref{th:alignment} with the experiment at initialization.
In Figure \ref{fig:diagonal_correlation_uua} we compute the diagonal correlation of the matrix $U_{J_L}^{\top}U_{A_{\l}}$. High correlation indicates that the coefficients are concentrated around the diagonal, which implies that $U_{J_L}$ and $U_{A_{\l}}$ have similar columns (up to sign), that is, the (left) singular vectors are approximately aligned. At initialization the alignment is most noticeable for low values of $\l$ as expected, while this effect is amplified and spread out to higher values of $\l$ after training.

% =========================================================
% ===================== discussion ======================
% =========================================================

\section{Discussion and Conclusion}

In this work we identified depth-scaling and spectral separation-induced alignment as relevant mechanisms for the study of Jacobians of neural networks. These properties arise naturally at initialization, and our experiments demonstrate that they might hold more generally during training. Under these assumptions, Proposition \ref{prop:singular_propto_fixed_gates} asymptotically recovers the singular value dynamics of balanced linear networks, suggesting that the implicit bias formally proved in the balanced linear case could be generalized to other architectures with the help of depth scaling and spectral separation.

% {\color{magenta}This is the part that should in my view be moved before Remark 7.2 or right after it.} \blue{Are you talking about the paragraph above or below?}
In our applications and experiments we have restricted ourselves to (single-mode) FGLNs; only in Remark \ref{rem:singular propto} did we work in a more general setting, but its conclusion lacks interpretability. One direction for future work is to study simple non-linear models, such as multi-modes FGLNs with controlled dependencies between modes.

It would also be of interest to rework Proposition \ref{prop:singular_propto_fixed_gates} under more experimentally sound assumptions. While the spectral separation hypothesis only holds for arbitrarily large $L$, the concentration of nonzero coefficients around the diagonal of $U_{J_L}^{\top}U_{A_{\l}}$ is a genuine phenomenon, which we described in Remark \ref{rem:first order alignment}. This effect could be leveraged to refine the singular value dynamics, yielding a more realistic but still tractable formulation.

\newpage

\section*{Impact Statement}
This paper presents work whose goal is to advance the field of machine learning. There are many potential societal consequences of our work, none of which we feel must be specifically highlighted here.

\section*{Acknowledgment}
This work was supported by ANR-22-CE23-0002 ERIANA and ANR-22-EXES-0009 MAIA, and and was granted access to the HPC resources of IDRIS under the allocation made by GENCI.

\bibliographystyle{icml2026}
\bibliography{references}

%%%%%%%%%%%%%%%%%%%%%%%%%%%%%%%%%%%%%%%%%%%%%%%%%%%%%%%%%%%%%%%%%%%%%%%%%%%%%%%
%%%%%%%%%%%%%%%%%%%%%%%%%%%%%%%%%%%%%%%%%%%%%%%%%%%%%%%%%%%%%%%%%%%%%%%%%%%%%%%
% APPENDIX
%%%%%%%%%%%%%%%%%%%%%%%%%%%%%%%%%%%%%%%%%%%%%%%%%%%%%%%%%%%%%%%%%%%%%%%%%%%%%%%
%%%%%%%%%%%%%%%%%%%%%%%%%%%%%%%%%%%%%%%%%%%%%%%%%%%%%%%%%%%%%%%%%%%%%%%%%%%%%%%
\newpage
\appendix
\onecolumn

\section{Proof of weak singular value dynamics on Fixed Gates Linear Networks}\label{app:gated_linear_network_dot_sk}

\begin{lemma}
    \label{lem:grad_W}
    Consider an FGLN $J(t)=W_L(t) D_{L-1}\cdots D_1 W_1(t)$ trained by gradient flow on the weights $(W_\ell(t))_{\ell=1}^L$ with loss $\loss(J)$. Define $M_\ell(t) \coloneq D_\ell W_\ell(t) D_{\l-1}$ for $\ell=1,\dots,L$.
    Then:
    $$\nabla_{W_{\l}}  \loss = D_{\l} M_{\l+1:L}^\top \nabla_J \loss(J) M_{1:\l-1}^\top $$
\end{lemma}

\begin{proof}
    Let $A = M_{\l+1:L} \, D_\l$ and $B = M_{1:\l-1}$ We write the network as
    \[J(W_{\l}) = A W_{\l} B.\]
    Denote $\langle X,Y\rangle = \mathrm{tr}(X^T Y)$ the Frobenius inner product. Let $H$ be any matrix. The chain rule gives
    \[\langle \nabla_{W_{\l}}\loss(J),H \rangle = \langle \nabla_{J}\loss(J), AHB \rangle = \langle A^{\top} \nabla_{J}\loss(J) B^{\top},H \rangle.\]
    Since this holds for all $H$, the equality follows.
\end{proof}

We may now prove Proposition \ref{prop:gated_linear_network_dot_sk}.

\begin{proof}[Proof of proposition \ref{prop:gated_linear_network_dot_sk}]
    We define the difference:
    \begin{equation*}
        \Delta_{\l}(t) \coloneqq M_{\l+1}^\top(t) M_{\l+1}(t) - M_{\l}(t) M_{\l}^\top(t)
    \end{equation*}
    We then show $\dot\Delta_{\l}(t) = 0$, which implies $\Delta_{\l}(t) = 0$ for all $t$. Indeed, observe that
    \begin{equation*}
        \dot{M}_{\l} = D_{\l}\dot W_{\l}D_{\l-1}
    \end{equation*}
    Develop $\dot\Delta_{\l}$:
    \begin{equation}\label{eq:delta_l dot}
        \dot\Delta_{\l} = \dot{M}_{\l+1}^\top M_{\l+1} + M_{\l+1}^\top\dot{M}_{\l+1} - \left( \dot{M}_{\l} M_{\l}^\top + M_{\l} \dot{M}_{\l}^\top\right)
    \end{equation}
    Each factors can be expanded with the help of Lemma \ref{lem:grad_W}:
    \begin{align*}
        \dot{M}_{\l+1}^\top M_{\l+1} &= -D_{\l}\left( D_{\l+1}M_{\l +2:L}^\top \nabla_J  \loss \, M_{1:\l}^\top\right)^{\top} D_{\l +1} D_{\l+1} W_{\l+1}D_{\l}\\
            &= -D_{\l}M_{1:\l} \nabla_J \loss^{\top} M_{\l +2:L} D_{\l +1}W_{\l+1}D_{\l}\\
            &= -M_{1:\l} \nabla_J \loss^{\top} M_{\l+1:L}
    \end{align*}
    Similarly,
    \begin{align*}
        M_{\l} \dot{M}_{\l}^\top &= -M_{1:\l} \nabla_J  \loss^{\top} M_{\l+1:L} \\
        M_{\l+1}^\top \dot{M}_{\l+1} &= -M_{\l+1:L}^\top \nabla_J \loss M_{1:\l}^\top \\
        \dot{M}_{\l} M_{\l}^\top &= -M_{\l+1:L}^\top \nabla_J \loss M_{1:\l}^\top.
    \end{align*}
    We see that all terms in equation \ref{eq:delta_l dot} cancel out. The rest is identical to the proof done by \cite{arora2019implicitregularizationdeepmatrix} in the deep linear case.
\end{proof}

\newpage
\section{Recollections on exterior powers}\label{app:Recollections on exterior powers}

In the course of the proof of Theorems \ref{th:singular value asymptotics} and \ref{th:alignment} in Appendices \ref{app:proof of spectrum at init} and \ref{app:proof of alignment} we rely on several occasions on the notion of exterior powers. We collect here their definition and fundamental and properties for the reader's convenience.

 From now-on fix $V$ a real vector space of dimension $n \ge 1$, which can be thought of as $\mathbb R^n$. Let $t \ge 1$ be an integer.

\subsection{Tensor powers}

Tensor powers are more common and easier to grasp than exterior powers. Although we do not use them in our proofs, we introduce them now to draw parallels later on.

\subsubsection{Description of tensor powers and tensors}

\begin{definition}\label{def:tensor power}
    The \emph{($t$-fold) tensor power} (or \emph{tensor product}) of $V$ is the vector space denoted $\bigotimes^t V$ generated by symbols
    $$v_1 \otimes \dots \otimes v_t, \hspace{1cm} v_1,\dots,v_t \in V$$
    subject to the following multi-linear relations: for every $1 \le i \le t$, $x_i,y_i,v_1,\dots,v_t \in V$, $a \in \mathbb R$,
    $$v_1 \otimes \dots \otimes (ax_i+y_i) \otimes \dots \otimes v_t = a(v_1 \otimes \dots \otimes x_i \otimes \dots \otimes v_t) + v_1 \otimes \dots \otimes y_i \otimes \dots \otimes v_t.$$
    Vectors of $\bigotimes^t V$ are often called \emph{tensors}.
\end{definition}

 Naturally, $\bigotimes^1 V = V$. The tensor product $\bigotimes^t V$ is \emph{generated} as a vector space by the tensors $v_1 \otimes \dots \otimes v_t$ called \emph{elementary tensors}. This implies that a general tensor in $\bigotimes^t V$ is a sum
$$\sum_i v_{1,i} \otimes \dots \otimes v_{t,i}.$$
Most tensors in $\bigotimes^t V$ cannot be simplified to an elementary tensor. For instance with $V = \mathbb R^2$ with standard basis $(e_1,e_2)$, $t=2$, the tensor $e_1 \otimes e_1 + e_2 \otimes e_2$ is not elementary.

\medskip

 The $\otimes$ symbol is not commutative: for $v,w \in V$ and $t=2$, the tensors $v \otimes w$ and $w \otimes v$ in $\bigotimes^2 V$ are different, unless $v=w$.

\medskip

 Fix a basis $(e_1,\dots,e_n)$ on $V$. Then $\bigotimes^t V$ is naturally equipped with a basis composed of the elementary tensors of the form
$$e_I \coloneqq e_{i_1} \otimes \dots \otimes e_{i_t}, \qquad I = (i_1,\dots,i_t) \subseteq \{1,\dots,n\}^t. $$
In particular, $\bigotimes^t V$ has dimension $n^t$. We may order the basis tensors $(e_I)$ using the lexicographic ordering.

\medskip

 If $V$ is moreover equipped with a dot product $\langle ,\rangle$, then so is $\bigotimes^t V$. It suffices to specify the dot product on elementary tensors, for which we set
$$\langle v_1 \otimes \dots \otimes v_t , w_1 \otimes \dots \otimes w_t \rangle \coloneqq \prod_{i=1}^t \langle v_i,w_i \rangle.$$
If a basis $(e_1,\dots,e_n)$ is orthonormal with respect to $\langle, \rangle$, then so is the induced basis $(e_I)$ on $\bigotimes^t V$ with respect to the induced dot product.

\subsubsection{Tensor powers of matrices}

Let $f$ be an endomorphism of $V$. Then we can define an endomorphism $\otimes^t f$ acting on $V$ by setting for every $v_1,\dots,v_t \in V$
$$(\otimes^t f)(v_1 \otimes \dots \otimes v_t) = f(v_1) \otimes \dots \otimes f(v_t).$$
By specifying a basis, we can thus make sense of the tensor powers of a matrix.

\begin{definition}
    Let $(e_i) = (e_1,\dots,e_n)$ the standard basis for $\mathbb R^n$, consider $(e_I)$ the induced basis on $\bigotimes^t\mathbb R^n$ ordered as above. Let $M$ be an $n \times n$ matrix, denote $f$ the automorphism of $\mathbb R^n$ it represents in the basis $(e_i)$. Then $\otimes^t M$ is the $n^t \times n^t$ matrix representing $\otimes^t f$ in the basis $(e_I)$ of $\bigotimes^t \mathbb R^n$.
\end{definition}

 If $M = (m_{i,j})$, then for $I,J \subseteq \{1,\dots,n\}^t$, $I = (i_1,\dots,i_t)$, $J=(j_1,\dots,j_t)$, the coefficient in position $(I,J)$ of $\otimes^t M$ is 
$$(\otimes^t M)_{I,J} = \prod_{k=1}^t m_{i_k,j_k}.$$
Tensor powers are compatible with matrix multiplication:
$$\otimes^t(MN) = (\otimes^tM)(\otimes^tN).$$

\medskip

 If $u_1,\dots,u_n$ denote the singular vectors of $M$ in $\mathbb R^n$, with corresponding singular values $s_1\ge \dots \ge s_n$, then the singular vectors of $\otimes^t M$ are exactly the elementary tensors
$$u_{i_1} \otimes \dots \otimes u_{i_t}, \hspace{1cm} I = (i_1,\dots,i_t) \subseteq \{1,\dots,n\}^t$$
with corresponding singular values given by the product $\prod_{i \in I} s_i$. In particular, (one of) the top singular vector(s) of $\otimes^tM$ is $u_1 \otimes \dots \otimes u_1$ with singular value $s_1^t$.

\subsection{Exterior powers}

We define and describe properties of exterior powers. The first two paragraphs are parallel to the ones on tensor powers. The third paragraph recalls the geometric meaning associated with exterior powers.

\subsubsection{Description of exterior powers and wedges}\label{subsubsection:Description of exterior powers and wedges}

\begin{definition}
    The \emph{($t$-fold) exterior power} (or \emph{wedge product}) of $V$ is the vector space denoted $\bigwedge^t V$ generated by symbols
    $$v_1 \wedge \dots \wedge v_t, \qquad v_1,\dots,v_t \in V$$
    subject to the same multi-linear relations as in Definition \ref{def:tensor power}, as well as the additional skew-symmetry relation: for every permutation of $\{1,\dots,n\}$ denoted $\sigma$,  for every $v_1,\dots,v_t \in V$,
    $$v_{\sigma(1)} \wedge \dots \wedge v_{\sigma(t)} = \varepsilon(\sigma)(v_1 \wedge \dots \wedge v_t)$$
    where $\varepsilon(\sigma)$ denotes the signature of the permutation $\sigma$. Vectors of $\bigwedge^t V$ are often called \emph{wedges}.
\end{definition}

 Naturally, $\bigwedge^1 V = V$. The wedge product $\bigwedge^t V$ is \emph{generated} as a vector space by the wedges $v_1 \wedge \dots \wedge v_t$ called \emph{elementary wedges}. This implies that a general wedge in $\bigwedge^t V$ is a sum
$$\sum_i v_{1,i} \wedge \dots \wedge v_{t,i}.$$
Most wedges in $\bigwedge^t V$ cannot be simplified to an elementary tensor. For instance with $V = \mathbb R^4$ with standard basis $(e_1,e_2,e_3,e_4)$, $t=2$, the wedge $e_1 \wedge e_2 + e_3 \wedge e_4$ is not elementary.

\medskip

 The $\wedge$ symbol is skew-commutative: for $v,w \in V$ and $t=2$, we have $v \wedge w = - w \wedge v$. Hence if $v = w$ then $v \wedge w = 0$. For a $t$-fold wedge $v_1 \wedge \dots \wedge v_t$, swapping any two components induces a change of sign.

\medskip

 Fix a basis $(e_1,\dots,e_n)$ on $V$. Then $\bigwedge^t V$ is naturally equipped with a basis composed of the elementary wedges of the form
$$e_I \coloneqq e_{i_1} \wedge \dots \wedge e_{i_t}, \hspace{1cm} I = (i_1 < \dots <i_t) \subseteq \{1,\dots,n\}^t. $$
In particular, $\bigwedge^t V$ has dimension $\binom{n}{t}$. If $t = n$, then $\bigwedge^t V$ is one-dimensional, and for $t > n$ we find $\bigwedge^t V = \{0\}$. We may order the basis wedges $(e_I)$ using the lexicographic ordering.

\medskip

 If $V$ is moreover equipped with a dot product $\langle ,\rangle$, then so is $\bigwedge^t V$. It suffices to specify the dot product on elementary wedges, for which we set
$$\langle v_1 \wedge \dots \wedge v_t , w_1 \wedge \dots \wedge w_t \rangle \coloneqq \det( \langle v_i,w_j \rangle)_{i,j}.$$
If a basis $(e_1,\dots,e_n)$ is orthonormal with respect to $\langle, \rangle$, then so is the induced basis $(e_I)$ on $\bigwedge^t V$ with respect to the induced dot product.

\subsubsection{Exterior powers of matrices}\label{subsubsection:Exterior powers of matrices}

Let $f$ be an endomorphism of $V$. Then we can define an endomorphism $\wedge^t f$ acting on $V$ by setting for every $v_1,\dots,v_t \in V$
$$(\wedge^t f)(v_1 \wedge \dots \wedge v_t) = f(v_1) \wedge \dots \wedge f(v_t).$$
By specifying a basis, we can thus make sense of the exterior powers of a matrix.

\begin{definition}
    Let $(e_i) = (e_1,\dots,e_n)$ the standard basis for $\mathbb R^n$, consider $(e_I)$ the induced basis on $\bigwedge^t\mathbb R^n$ ordered as above. Let $M$ be an $n \times n$ matrix, denote $f$ the automorphism of $\mathbb R^n$ it represents in the basis $(e_i)$. Then $\wedge^t M$ is the $\binom{n}{t} \times \binom{n}{t}$ matrix representing $\wedge^t f$ in the basis $(e_I)$ of $\bigwedge^t \mathbb R^n$.
\end{definition}

 If $M = (m_{i,j})$, then for $I,J \subseteq \{1,\dots,n\}^t$, $I = (i_1<\dots<i_t)$, $J=(j_1<\dots<j_t)$, the coefficient in position $(I,J)$ of $\wedge^t M$ is
$$(\wedge^t M)_{I,J} = \det M^{I,J}$$
where $M^{I,J}$ denotes the sub-matrix of $M$ keeping only the lines (resp. columns) indexed by $I$ (resp. $J$). Exterior power is compatible with matrix multiplication:
$$\wedge^t(MN) = (\wedge^tM)(\wedge^tN).$$

\medskip

 If $u_1,\dots,u_n$ denote the singular vectors of $M$ in $\mathbb R^n$, with corresponding singular values $s_1,\dots,s_n$, then the singular vectors of $\wedge^t M$ are exactly the elementary wedges
$$u_{i_1} \wedge \dots \wedge u_{i_t}, \hspace{1cm} I = (i_1<\dots<i_t) \subseteq \{1,\dots,n\}^t$$
with corresponding singular values given by the product $\prod_{i \in I} s_i$. In particular, (one of) the top singular vector(s) of $\wedge^tM$ is $u_1 \wedge u_2 \wedge \dots \wedge u_t$ with singular value $s_1s_2\dots s_t$.

\begin{remark}\label{rem:exterior powers involved}
    This last observation is the reason why exterior powers are naturally involved as soon as one is interested in singular values of a matrix $M$ other than the top one.
\end{remark}

\subsubsection{Geometric meaning of exterior powers}\label{subsubsection:Geometric meaning of exterior powers}

Intuitively the exterior power $\bigwedge^t V$ should be thought of as the space generated by $t$-dimensional linear subspaces of $V$. In particular, if $v_1,\dots,v_t \in V$ are linearly independent one should think about the elementary wedge $v_1 \wedge \dots \wedge v_t$ as representing the subspace spanned by $v_1,\dots,v_t$.

\medskip

 Let $v_1,\dots,v_t$ (resp. $w_1,\dots,w_t$) be linearly independent vectors in $V$. Denote $E$ (resp. $F$) the subspace spanned by the $v_i$ (resp. $w_i$). Then
\begin{center}
    $v_1 \wedge \dots \wedge v_t = a(w_1 \wedge \dots \wedge w_t)$ for some scalar $a \in \mathbb R$ if and only if $E = F$.
\end{center}
If $V$ is equipped with a dot product $\langle,\rangle$ then the induced dot product on $\bigwedge^t V$ corresponds to the intuitive notion of angle between $t$-dimensional subspaces. In particular:
\begin{center}
    $\langle v_1 \wedge \dots \wedge v_t, w_1 \wedge \dots \wedge w_t \rangle = 0$ if and only if $E$ and $F$ are at right angle.
\end{center}
Suppose that each $v_i=v_{i,L}$ depends on some integer $L$. Assume each $v_{i,L}$ and $w_i$ is a unit vector. Setting $E_L = \mathrm{Span}(v_{1,L},\dots,v_{t,L})$, then
\begin{center}
    $v_{1,L} \wedge \dots \wedge v_{t,L} \xrightarrow{L \to \infty} \pm(w_1 \wedge \dots \wedge w_t)$ if and only if $E_L$ aligns with $F$ as $L \to \infty$.
\end{center}

\newpage
\section{Proof of theorem \ref{th:singular value asymptotics}}\label{app:proof of spectrum at init}

We divide the proof in two steps. First we show the existence and formula for the Lyapunov exponents $\gamma_i$ using standard results from the literature of products random matrices. Then we compute a second order term for the expectation $\mathbb{E}[\log s_{i,L}]$; this is achieved by studying the quotient of operator norms $\|AB\|/\|A\|\|B\|$ when $A$ and $B$ are random square matrices with spectral separation.  

\subsection{First half of the proof of Theorem \ref{th:singular value asymptotics}}

For the rest of this paragraph we fix $1 \le r \le n$, $0<p<1$ and $\sigma > 0$. We seek to prove the existence and value for the Lyapunov exponents $\gamma_i$ in Theorem \ref{th:singular value asymptotics}, which already implies Corollary \ref{cor:exp spectrum for layers}.

\medskip

 We begin with a theorem generalizing the statements from \cite{bougerol2012products}, Example I.2.3, Theorem I.4.1.

\begin{theorem}\label{th:exp spectrum of iterated product}
    Let $(Y_i)$ be a sequence of i.i.d. random matrices. For any positive integer $L$, define the random matrix
    $$J_L = Y_L \dots Y_1.$$
    Assume that
    \begin{itemize}
        \item The expectation $\mathbb E[\log^+(\|Y_1\|)]$ is finite, and
        \item almost surely, $J_L$ is nonzero for every $L$.
    \end{itemize} 
    Denote $s_{1,L} \ge \dots \ge s_{n,L}$ the singular values of $J_L$. Then there exists numbers $\gamma_1 \ge \dots \ge \gamma_n$ (possibly $-\infty$) such that for any $1 \le i \le n$, almost surely
    $$\frac{1}{L} \log s_{i,L} = \gamma_i.$$
    Moreover, if for any orthogonal matrix $U$, $Y_1U$ has same distribution as $Y_1$, then for any $1 \le i \le n$
    $$\gamma_1 + \dots + \gamma_i = \mathbb E[\log \| Y_1 e_1 \wedge \dots \wedge Y_1 e_i \|].$$
\end{theorem}

 The original statements in \cite{bougerol2012products} assumes that the matrices $Y_i$ are invertible. The reason for this is to ensure that the product $J_L$ is nonzero so that one can divide by the norm $\|J_L\|$ to define \emph{cocycles}. We reviewed their proof and noticed that it suffices to assume (up to a slight generalization of the notion of cocycle) that $J_L$ is almost surely nonzero for every $L$. Unfortunately, reproducing and adapting their argument here would be excessively long, and, to the best of our knowledge, there is no other source that proves Theorem \ref{th:exp spectrum of iterated product} under our hypotheses.

\medskip

 The existence part in Theorem \ref{th:singular value asymptotics} follows from Theorem \ref{th:exp spectrum of iterated product} once we check the assumptions holds.

\begin{lemma}\label{lem:checking assumptions}
    A sequence of i.i.d. $(r,p,\sigma)$-layers $(Y_i)$ satisfies the assumptions of Theorem \ref{th:exp spectrum of iterated product}.
\end{lemma}
\begin{proof}
    Write $Y_1 = D_1W_1$. For any orthogonal matrix $U$, the product $W_1 U$ has same distribution as $W_1$, so the same holds true of $Y_1$. It is known that $\mathbb E [\log^+ \|W_1\|]$ is finite, thus
    $$0 \le \mathbb E [\log^+ \|Y_1\|] \le \mathbb E [\log^+ \|D_1\| \cdot \| W_1 \|] = \mathbb E [\log^+ \|W_1\|] < \infty.$$
    We show that $J_L$ is nonzero for every $L$ almost surely by showing inductively that it has rank at least $r$. This is clearly the case of $J_1$. If $J_L$ has rank $\ge r$, then so does $W_{L+1} J_L$ almost surely. The rank of $J_{L+1} = D_{L+1}W_{L+1}J_L$ is given by
    $$\mathrm{rk} (D_{L+1}) - \dim (\mathrm{Im}(W_{L+1}J_L) \cap \ker D_{L+1})$$
    Recall that the dimension of the intersection of independent uniformly distributed subspaces of dimensions $d_1$ and $d_2$ is almost surely $\max(0,n-d_1-d_2)$, so the dimension of $\ker J_{L+1}$ is almost surely
    $$\mathrm{rk}(J_{L+1}) = \mathrm{rk} (D_{L+1}) - \max(0,-\mathrm{rk} (J_L) + \mathrm{rk} (D_{L+1})) = \min(\mathrm{rk} (D_{L+1}),\mathrm{rk} (J_{L})) \ge r.$$
\end{proof}

 We now compute the Lyapunov exponents. The appearance of wedge products in the formula 
$$\gamma_1 + \dots + \gamma_i = \mathbb E[\log \| Y_1 e_1 \wedge \dots \wedge Y_1 e_i \|]$$
given in Theorem \ref{th:exp spectrum of iterated product} is explained by Remark \ref{rem:exterior powers involved}. We compute this expectation in two steps.

\begin{lemma}
    Let $W$ be a $\sigma$-Ginibre matrix. For $1 \le t \le n$, let $D$ be diagonal with entries in $\{0,1\}$, of rank $t$ (i.e. there are $t$ copies of $1$ along the diagonal). Set $Y = DW$. Then for any $1 \le i \le n$,
    $$\mathbb E[\log \| Y e_1 \wedge \dots \wedge Y e_i \|] =
    \begin{cases}
        \frac{i}{2} \ln (2\sigma^2) + \frac{1}{2}\sum_{k=1}^i \psi\left( \frac{t-k+1}{2} \right) & \text{if } i \le r \\
        - \infty & \text{if } i > r
    \end{cases}$$
    where $\psi$ denotes the digamma function.
\end{lemma}
\begin{proof}
    Up to relabeling, we may assume that
    $$D = \mathrm{diag}(1,\dots,1,0,\dots,0).$$
    Observe that $We_j$ is the $j$-th column of $W$, hence $Ye_j$ is a vector whose first $t$ coordinates coincide with those of $We_j$, and the rest are zeros. Thus we can view $(Ye_j)_{1 \le j \le i}$ as a collection of independent Gaussian vectors in $\mathbb R^t$. If $i>t$, this collection is not linearly independent, hence
    $$Y e_1 \wedge \dots \wedge Y e_i = 0.$$
    From now on we assume $i \le t$. Denoting $\langle ,\rangle$ the standard scalar product in $\mathbb{R}^t$, the norm on the wedge product $\bigwedge^i \mathbb R^t$ satisfies (see paragraph \ref{subsubsection:Description of exterior powers and wedges})
    $$\|Ye_1 \wedge \dots \wedge Ye_i \|^2 = \det(\langle Ye_p,Ye_q \rangle)_{1\le p,q \le i} = \det(X^{\top} X)$$
    where $X$ is the $t \times i$ matrix whose columns are the independent Gaussian vectors $Ye_1,\dots,Ye_i$ viewed in $\mathbb R^t$. The $i \times i$ matrix $X^{\top}X$ thus follows the Wishart distribution, it is known that the log-expectation of its determinant is
    $$\mathbb E [\log \det(X^{\top} X)] = i \ln (2\sigma^2) + \sum_{k=1}^i \psi\left( \frac{t-k+1}{2} \right).$$
\end{proof}

\begin{lemma}
    Let $Y_1=D_1W_1$ be an $(r,p,\sigma)$-layer. For any $1 \le i \le n$, denote $\lambda_i = \mathbb E[\log \| Y_1 e_1 \wedge \dots \wedge Y_1 e_i \|]$. Then
    $$\lambda_i =
    \begin{cases}
        - \infty & \text{if } i > r \\
        \frac{i}{2}\ln(2\sigma^2) + \frac{\sum_{t=r}^n \left[ \binom{n}{t}p^t(1-p)^{n-t} \sum_{k=1}^i \psi\left( \frac{t-k+1}{2} \right)\right]}{2\sum_{m=r}^n \binom{n}{m} p^m (1-p)^{n-m}} & \text{otherwise}.
    \end{cases}$$
    where $\psi$ denotes the digamma function. For $i \le r \ll np$ the expression simplifies to
    $$\lambda_i \approx \frac{i}{2}\ln(2\sigma^2) + \frac{1}{2}\sum_{t=r}^n \left[ \binom{n}{t}p^t(1-p)^{n-t} \sum_{k=1}^i \psi\left( \frac{t-k+1}{2} \right)\right]$$
\end{lemma}
\begin{proof}
    We split the expectation according to the rank of $D_1$. If $i>r$, some terms are $- \infty$, hence $\lambda_i = -\infty$. From now-on we assume $i \le r$.

    \smallskip
     By definition for $r \le t \le n$
    $$P(\mathrm{rk}D_1 = t) = \frac{P(\mathrm{rk}D = t)}{P(\mathrm{rk}D \ge r)} = \frac{\binom{n}{t}p^t(1-p)^{n-t}}{\sum_{m=r}^n \binom{n}{m} p^m (1-p)^{n-m}}$$
    where $D$ denotes a $(p)$-gate, whose rank follows an $(n,p)$-binomial distribution. Hence, letting $\lambda_i = \mathbb E[\log \| Y_1 e_1 \wedge \dots \wedge Y_1 e_i \|]$, we find
    $$\lambda_i = \frac{i}{2}\ln(2\sigma^2) + \frac{\sum_{t=r}^n \left[ \binom{n}{t}p^t(1-p)^{n-t} \sum_{k=1}^i \psi\left( \frac{t-k+1}{2} \right)\right]}{2\sum_{m=r}^n \binom{n}{m} p^m (1-p)^{n-m}}$$
    When $r \ll np$, the sum in the denominator is almost indistinguishable from 1, hence the approximation.
\end{proof}

 We deduce the values of the exponents $\gamma_i$.

\begin{lemma}
    Let $Y_1=D_1W_1$ be an $(r,p,\sigma)$-layer. For any $1 \le i \le r$, denote $\lambda_i = \mathbb E[\log \| Y_1 e_1 \wedge \dots \wedge Y_1 e_i \|]$ and $\gamma_i \coloneqq \lambda_{i} - \lambda_{i-1}$. Then
    \begin{align*}
        \gamma_i & = \ln(\sqrt{2} \sigma) + \frac{\sum_{t=r}^n \left[ \binom{n}{t}p^t(1-p)^{n-t} \psi\left( \frac{t-i+1}{2} \right)\right]}{2\sum_{m=r}^n \binom{n}{m} p^m (1-p)^{n-m}} \\
                & \approx \ln(\sqrt{2} \sigma) + \frac{1}{2}\sum_{t=r}^n \left[ \binom{n}{t}p^t(1-p)^{n-t} \psi\left( \frac{t-i+1}{2} \right)\right] & \text{for } r \ll np.
    \end{align*}
    where $\psi$ denotes the digamma function. Moreover, $\gamma_i > \gamma_{i+1}$ for $i<r$.
\end{lemma}
\begin{proof}
    The value for $\gamma_i$ and its approximation follow from those of $\lambda_i$. If $i < r$, as $\psi$ is increasing, we find $\gamma_i > \gamma_{i+1}$.
\end{proof}

\subsection{Second half of the proof of Theorem \ref{th:singular value asymptotics}}

Again we fix $1 \le r \le n$, $0<p<1$ and $\sigma > 0$. We wish to compute the limit of the quantity $\mathbb E \left[\log(s_{i,L}) \right] - \gamma_i L$ as $L \to \infty$.

\medskip

 Let us motivate the Lemma \ref{lem:log-expect of submultiplicative} below. If $A$ and $B$ are two square matrices, then
$$\|AB\| \le \|A\| \|B\|.$$
In general the gap between both quantities can be arbitrary, as it heavily depends on the relative orientation of singular vectors of $A$ and $B$. However if $A$ and $B$ exhibit spectral separation and Haar-distributed singular vectors, then it is possible to estimate on average the value of the ratio $\frac{\|AB\|}{\|A\| \|B\|}$. This is a particular case of the content of Lemma \ref{lem:log-expect of submultiplicative}.

\begin{lemma}\label{lem:log-expect of submultiplicative}
    Let $(A_L)_L,(B_L)_L$ be sequences of random matrices. For every $1 \le i \le n$, denote $s_{i,L}^A$ (resp. $s_{i,L}^B$) the $i$-th singular value of $A_L$ (resp. $B_L$), in decreasing order. Assume further that
    \begin{enumerate}
        \item[(a)] for $1 \le i \le r$, almost surely $s_{i,L}^A \neq 0$ and $\frac{s_{i+1,L}^{A}}{s_{i,L}^A} \xrightarrow{L \to \infty} 0$, similarly for $B$,
        \item[(b)] writing $A_L = U_L^A S_L^A V_L^A$ and $B_L = U_L^B S_L^B V_L^B$ for the respective SVDs, $V_L^AU_L^B$ is Haar distributed and independent of $(S_L^A,S_L^B)$. 
    \end{enumerate}
    Then for every $1 \le t \le r$,
    $$\mathbb E \left[ \log \frac{\|\wedge^t(A_LB_L)\|}{\|\wedge^t A_L\| \|\wedge^t B_L\|} \right] \xrightarrow{L \to \infty} \mathbb{E}[\log |\det \Omega^{t,t}|] \le 0$$
    where $\Omega^{t,t}$ is the $t \times t$ top-left block of a Haar-distributed orthogonal matrix.
\end{lemma}

 For the notions and properties of exterior powers of matrices in the statement and proof of Lemma \ref{lem:log-expect of submultiplicative}, see paragraph \ref{subsubsection:Exterior powers of matrices}.

\begin{proof}
    Fix some integer $L$. It is harmless to multiply $A_L$ on the left (resp. $B_L$ on the right) by any random orthogonal matrix. Thus we can assume that $A_L = S_L^A V_L^A$, $B_L = U_L^BS_L^B$. Define
    $$\Omega = V_L^A U_L^B$$
    which depends on $L$ (but we omit the subscript for simplicity). By definition
    $$A_LB_L = S_L^A \Omega S_L^B.$$
    We are interested in the log-expectation
    $$\mathbb E \left[\log \frac{\|\wedge^t(A_LB_L)\|}{\|\wedge^t A_L\| \|\wedge^t B_L\|}\right] = \mathbb E \left[\log \frac{\|(\wedge^t S_L^A) \times (\wedge^t \Omega) \times (\wedge^t S_L^B)\|}{\|\wedge^t S_L^A\| \|\wedge^t S_L^B\|}\right]$$
    with
    $$\wedge^t S_L^A = \mathrm{diag}\left( \prod_{i \in I} s_{i,L}^A \right)_I, \; \; \wedge^t \Omega = (\det\Omega^{I,J})_{I,J}, \; \; \wedge^t S_L^B = \mathrm{diag}\left( \prod_{j \in J} s_{j,L}^B \right)_J.$$
    Let $I_0 = \{1,\dots,t\}$ which is the leading index for the standard basis of $\bigwedge^t \mathbb R^n$. Then
    $$\|\wedge^t S_L^A\| = \prod_{i \in I_0} s_{i,L}^A, \; \; \|\wedge^t S_L^B\| = \prod_{i \in I_0} s_{i,L}^B$$
    hence, by assumption \textit{(a)} the diagonal matrices $S_L^A / \|\wedge^t S_L^A\|$ and $S_L^B / \|\wedge^t S_L^B\|$ almost surely exist and converge to 
    $$E_{I_0,I_0} \coloneqq \mathrm{diag}(1,0\dots,0).$$
    By assumption \textit{(b)} the matrix $\Omega$ is independent of $S_L^A$ and $S_L^B$ and is identically distributed as $L$ varies, thus we may assume that it does not depend on $L$: the log-expectation above remains unchanged for every $L$. In this setting, almost surely
    $$\frac{1}{\|\wedge^t A_L\| \|\wedge^t B_L\|}\wedge^t(A_LB_L) \xrightarrow{L \to \infty} E_{I_0,I_0} \times \wedge^t \Omega \times E_{I_0,I_0}.$$
    Observe that the matrix on the right has vanishing coefficients everywhere except at index $(I_0,I_0)$ where it evaluates to $\det \Omega^{I_0,I_0}$:
    $$\|E_{I_0,I_0} \times \wedge^t \Omega \times E_{I_0,I_0}\| = |\det \Omega^{I_0,I_0}|.$$
    We deduce the convergence of the log-expectation via the dominated convergence theorem, with domination
    $$\log |\det \Omega^{I_0,I_0}|\le \log \frac{\|\wedge^t(A_LB_L)\|}{\|\wedge^t A_L\| \|\wedge^t B_L\|} \le 0.$$
    where the inequality on the left follows from $\|\cdot\|_{\infty} \le \|\cdot\|$ as matrix norms. We have shown
    $$\mathbb E \left[\log \frac{\|\wedge^t(A_LB_L)\|}{\|\wedge^t A_L\| \|\wedge^t B_L\|}\right] \xrightarrow{L \to \infty} \mathbb{E}[\log |\det \Omega^{I_0,I_0}|]$$
    which concludes by observing that the variable $\Omega^{t,t}$ from the statement has same distribution as $\Omega^{I_0,I_0}$.
\end{proof}

 We make use of Lemma \ref{lem:log-expect of submultiplicative} as follows. With notations of Theorem \ref{th:exp spectrum of iterated product} it follows from the almost sure convergence and the dominated convergence theorem that
$$\gamma_1 + \dots + \gamma_i = \lim_{L \to \infty} \frac{1}{L} \mathbb{E}[\log (s_{1,L} \dots s_{i,L})] = \lim_{L \to \infty} \frac{1}{L} \mathbb{E}[\log\|\wedge^i J_L\|].$$
Set $u_L = \mathbb{E}[\log\|\wedge^i J_L\|]$ and write
$$u_{p+q} + d(p,q) = u_p + u_q$$
with by definition
$$d(p,q) = - \mathbb E \left[ \log \frac{\|\wedge^i(J_p J_q)\|}{\|\wedge^iJ_p\| \|\wedge^iJ_q\|} \right].$$
By Corollary \ref{cor:exp spectrum for layers} and properties of $\sigma$-Ginibre matrices the assumptions of Lemma \ref{lem:log-expect of submultiplicative} are satisfied in the setting of Theorem \ref{th:singular value asymptotics}, thus
$$d(p,q) \xrightarrow{(p,q) \to \infty} - \mathbb E [\log |\det\Omega_i |] \coloneqq d_i \ge 0.$$
From Lemma \ref{lem:subadditive asymptotics} below we find
$$\frac{1}{L}u_L = \gamma_1+\dots+\gamma_i + \frac{d_i}{L} + o \left( \frac{1}{L}\right)$$
Subtracting consecutive values of $i$, we obtain the final statement of Theorem \ref{th:singular value asymptotics}.

\begin{lemma}\label{lem:subadditive asymptotics}
    Let $(u_L)$ be a sequence of real numbers satisfying for all $p,q \ge 1$
    $$u_{p+q} + d(p,q) = u_p + u_q$$
    with $d(p,q) \ge 0$ and $d(p,q) \xrightarrow{(p,q) \to \infty} d$. Then the sequence $(u_L/L)$ admits a finite limit $u$, and
    $$\frac{1}{L}u_L = u + \frac{d}{L} + o \left( \frac{1}{L}\right).$$
\end{lemma}
\begin{proof}
    Replacing $u_L$ with $u_L + (L-1)d$, we may assume $d=0$. By Fekete's subadditive lemma, the sequence $(u_L/L)$ converges in $[-\infty,+\infty[$, let $u$ be its limit. Fix some integer $L \ge 2$, induction shows that for every $m \ge 1$,
    $$\frac{1}{L^m}u_{L^m} + (L-1)\sum_{k=2}^m \frac{\frac{1}{L-1}\sum_{i=1}^{L-1} d(L^{k-1},iL^{k-1})}{L^k} = \frac{1}{L} u_L.$$
    We begin by showing that the limit $u$ is finite. By assumption on $d(\cdot,\cdot)$, for any integer $L$ and $2 \le i < L$ the quantity $d(L^{k-1},iL^{k-1})$ is bounded as a function of $k$. Hence the same can be said of the average
    $$\frac{1}{L-1}\sum_{i=1}^{L-1} d(L^{k-1},iL^{k-1})$$
    which ensures the convergence of the series
    $$0 \le \sum_{k=2}^{\infty} \frac{\frac{1}{L-1}\sum_{i=1}^{L-1} d(L^{k-1},iL^{k-1})}{L^k} < \infty$$
    which in turn shows that the limit $u = \lim_{m \to \infty} \frac{1}{L^m} u_{L^m}$ is finite. We have thus established that for any $L\ge 2$,
    $$\frac{1}{L}u_L = u + (L-1)\sum_{k=2}^{\infty} \frac{\frac{1}{L-1}\sum_{i=1}^{L-1} d(L^{k-1},iL^{k-1})}{L^k}$$
    Consider $\varepsilon >0$ and $L_0$ such that for any $L\ge L_0$ and $1 \le i < L$, $0 \le d(L^{k-1},iL^{k-1}) < \varepsilon$. Hence for $L \ge L_0$
    $$\left|\frac{1}{L}u_L - u\right| \le \varepsilon \sum_{k=2}^{\infty} \frac{L-1}{L^k} = \frac{\varepsilon}{L}.$$
    Thus $L(\frac{1}{L}u_L - u) \xrightarrow{L \to \infty} 0$ which is what we needed to show.
\end{proof}

\newpage
\section{Proof of theorem \ref{th:alignment}}\label{app:proof of alignment}

\subsection{Another statement}

Theorem \ref{th:alignment} is a consequence of the slightly more general statement.

\begin{theorem}\label{th:pre-alignment}
    Let $(S_L)$ be a sequence of diagonal matrices and let $C=(c_{i,j})$ be a symmetric matrix such that
    \begin{enumerate}
        \item[(i)] The diagonal coefficients of $S_L$ are non-negative non-increasing: $s_{1,L} \ge \dots \ge s_{n,L} \ge 0$,
        \item[(ii)] (spectral separation) There is an integer $1 \le r \le n-1$ such that
        $$\forall 1 \le i \le r, \; \; \frac{s_{i+1,L}}{s_{i,L}} \xrightarrow{L \to \infty} 0,$$
        \item[(iii)] For every $1 \le i \le r$, if $C^{i,i}$ denotes the top-left $i \times i$ sub-matrix of $C$, then $C^{r,r}$ is positive definite.
    \end{enumerate}
    For each integer $L$, write the SVD of $X_L \coloneqq S_L C S_L$ as
    $$X_L = U_L \Sigma_L U_L^{\top}, \; \; \Sigma_L = \mathrm{diag}(\sigma_{1,L},\dots,\sigma_{n,L}), \; \; |\sigma_j| \ge |\sigma_{j+1}| \; \;\forall 1 \le j \le n$$
    
     Denote by $e_1,\dots,e_n$ the standard basis of $\mathbb R^n$ and by $u_{1,L},\dots,u_{n,L}$ the columns of $U_L$ (which coincide with the singular vectors of $X_L$). We can assume $(e_i^{\top} u_{i,L}) \ge 0$ for every $1\le i \le n$.
    
     Then for all $1 \le i \le r$:
    \begin{enumerate}
        \item[(a)] The singular vector $u_i$ aligns with $e_i$, i.e. $u_i \xrightarrow{L \to \infty} \pm e_i$,
        \item[(b)] With the convention that $C^{0,0} = 1$, then
        $$\frac{\sigma_{i,L}}{s_{i,L}^2} \xrightarrow{L \to \infty} \frac{\det C^{i,i}}{\det C^{i-1,i-1}} \eqcolon \ell_i$$
        \item[(c)] For all $1 \le j \le r$, $j \neq i$,
        $$(e_j^{\top} u_{i,L}) = 
        \begin{cases}
            (T_C^{-1})_{i,j} \frac{s_{j,L}}{s_{i,L}} + o\left(\frac{s_{j,L}}{s_{i,L}}\right) & \text{if } j<i, \\
            (T_C)_{j,i} \frac{s_{i,L}}{s_{j,L}} + o\left(\frac{s_{i,L}}{s_{j,L}}\right) & \text{if } i<j.
        \end{cases}$$
        where $T_C$ is the $r \times r$ lower-triangular matrix with all-$1$ diagonal such that
        $$C^{r,r} = T_C \Sigma_C {T_C}^{\top}$$
        is the Cholesky decomposition of $C^{r,r}$ (in which case $\Sigma_C = \mathrm{diag}(\ell_1,\dots,\ell_r)$).
    \end{enumerate}
    Hence the top-left $r \times r$ sub-matrix $U_L^{r,r}$ converges to the identity matrix, with off-diagonal coefficients decaying faster the further away they are from the diagonal.
\end{theorem}

\begin{remark}
    With notations of Theorem \ref{th:pre-alignment}, statements $(b)$ and $(c)$ are equivalent to
    \begin{align*}
        (S_L^{r,r})^{-1} \Sigma_L (S_L^{r,r})^{-1} & \xrightarrow{L \to \infty} \Sigma_C \\
        (S_L^{r,r})^{-1} U_L^{r,r} S_L^{r,r} & \xrightarrow{L \to \infty} T_C, \\
        S_L^{r,r} U_L^{r,r} (S_L^{r,r})^{-1} & \xrightarrow{L \to \infty} (T_C^{-1})^{\top}.
    \end{align*}
\end{remark}

\begin{remark}
    Statement $\emph{(c)}$ of Theorem \ref{th:pre-alignment} can be made slightly stronger. Even for $1 \le i \le r < j$, there is a constant $K_{j,i}$ such that
    $$(e_j^{\top} u_{i,L}) = K_{j,i}\frac{s_{i,L}}{s_{j,L}} + o\left(\frac{s_{i,L}}{s_{j,L}}\right)$$
    but this time the connection between $C$ and $K_{j,i}$ is less clear.
\end{remark}

\subsection{Proof of Theorem \ref{th:pre-alignment}}

 We begin by recalling a standard statement from matrix perturbation theory, the proof of which we omit.

\begin{lemma}\label{lem:rank 1 matrix perturbation}
    Let $(M_L)$ be a sequence of symmetric matrices. Let $\sigma_L$ be the largest (in absolute value) eigenvalue of $M_L$, and $u_L$ be a corresponding normalized eigenvector. Assume that
    $$M_L \xrightarrow{L \to \infty} e_1e_1^{\top}.$$
    Then $\sigma_L \xrightarrow{L \to \infty} 1$, $u_L \xrightarrow{L \to \infty} \pm e_1$, and all other eigenvalues of $M_L$ converge to zero.
\end{lemma}

 From this we deduce by induction:
\begin{lemma}\label{lem:alignment (a) (b)}
    Statements \emph{(a)} and \emph{(b)} of Theorem \ref{th:pre-alignment} hold. Moreover, $\frac{\sigma_{r+1,L}}{\sigma_{r,L}} \xrightarrow{L \to \infty} 0$.
\end{lemma}
\begin{proof}
    For the base case, notice that the matrix $\frac{1}{s_{1,L}^2}X_L$ converges to $e_1e_1^{\top}$. By Lemma \ref{lem:rank 1 matrix perturbation}
    $$u_1 \xrightarrow{L \to \infty} \pm e_1, \; \; \frac{\sigma_{1,L}}{s_{1,L}^2} \to c_{1,1} = \det C^{1,1} = \ell_1.$$
    Assume the statements have been established up to $1 \le i-1 < r$. Consider now the matrix
    $$\wedge^{i} (X_L) = (\wedge^{i}S_L)(\wedge^i C) (\wedge^i S_L)$$
    with top-left coefficient given by $\eta_L = \det C^{i,i} \prod_{k=1}^i s_{k,L}^2$. The spectral separation assumption implies
    $$\frac{1}{\eta_L} (\wedge^i X_L) \to (e_1 \wedge \dots \wedge e_i) (e_1 \wedge \dots \wedge e_i)^{\top}$$
    which shows that by Lemma \ref{lem:rank 1 matrix perturbation} again
    $$\frac{\prod_{k=1}^i \sigma_{k,L}}{\prod_{k=1}^i s_{k,L}^2} =\prod_{k=1}^i \frac{\sigma_{k,L}}{s_{k,L}^2} \to \det C^{i,i}, \; \; u_{1,L} \wedge \dots \wedge u_{i,L} \to \pm e_1 \wedge \dots \wedge e_i.$$
    By induction the left-hand side yields $\sigma_{i,L} / s_{i,L}^2 \to \ell_i$. Moreover, according to paragraph \ref{subsubsection:Geometric meaning of exterior powers} the subspace $\mathrm{Span}(u_{1,L},\dots,u_{i,L})$ aligns with  $\mathrm{Span}(e_1,\dots,e_i)$. By induction, this was already the case for $i-1$. As $u_{i,L}$ is orthogonal to $\mathrm{Span}(u_{1,L},\dots,u_{i-1,L})$, it aligns with the orthogonal complement of $\mathrm{Span}(e_1,\dots,e_{i-1})$ in $\mathrm{Span}(e_1,\dots,e_{i})$, hence $u_{i,L} \to \pm e_i$.

    For the remaining claim, consider the case $i=r$. The second highest eigenvalue of $\frac{1}{\eta_L}(\wedge^r X_L)$ is precisely $\sigma_{r+1,L}/\sigma_{r,L}$, which converges to zero according to Lemma \ref{lem:rank 1 matrix perturbation}.
\end{proof}

 Lemma \ref{lem:alignment (a) (b)} establishes in particular that
\[(S_L^{r,r})^{-1} \Sigma_L (S_L^{r,r})^{-1} \xrightarrow{L \to \infty} \Sigma_C.\]
We may now prove the remaining statement. 

\begin{lemma}\label{lem:alignment (c)}
    Statement \emph{(c)} of Theorem \ref{th:pre-alignment} holds.
\end{lemma}
\begin{proof}
    Write $C^{r,r} = T_C \Sigma_C T_{C}^{\top}$ for the Cholesky decomposition as in the statement. By definition of $X_L$, we have
    \begin{equation}\label{eq:C=SLXLSL-1}
        C = S_L^{-1} X_L S_L^{-1} = S_L^{-1}U_L \Sigma_L U_L^{\top} S_L^{-1} = (S_L^{-1} U_L S_L)(S_L^{-1} \Sigma_L S_L^{-1})(S_L U_L^{\top} S_L^{-1}).
    \end{equation}
    Let $P_L = S_L^{-1} U_L S_L$, which we write in blocks as
    \[P_L =
    \begin{pmatrix}
        P_L^{r,r} & \varepsilon_L \\
        * & *
    \end{pmatrix}.\]
    The spectral separation assumption implies that the $r \times (n-r)$ matrix $\varepsilon_L$ converges to zero. Decomposing the $r \times r$ matrix $P_L^{r,r}$ as
    \[P_L^{r,r} = Q_L + R_L\]
    where $Q_L$ is lower-triangular with all-1 diagonal and $R_L$ is upper triangular with all-0 diagonal, spectral separation shows once more that $R_L$ converges to 0. A computation performed in Lemma \ref{lem:lemma in alignment (c)} below yields
    \[C^{r,r} = Q_L \Sigma_CQ_L^{\top} + o(1).\]
    In particular, $Q_L \Sigma_CQ_L^{\top}$ is the Cholesky decomposition for the positive definite matrix $C^{r,r} + o(1)$, which converges to the positive definite matrix $C^{r,r}$. By continuity and uniqueness of the Cholesky decomposition for positive definite matrices, we find
    \[Q_L \xrightarrow{L \to \infty} T_C\]
    and hence $P_L^{r,r}$ converges to $T_C$. Thus
    \[(S_L^{r,r})^{-1} U_L^{r,r} S_L^{r,r} \xrightarrow{L \to \infty} T_C\]
    and the remaining convergence is obtained by taking transpose and inverse.
\end{proof}

\begin{lemma}\label{lem:lemma in alignment (c)}
    With notations of the proof of Lemma \ref{lem:alignment (c)},
    \[C^{r,r} = Q_L \Sigma_CQ_L^{\top} + o(1).\]
\end{lemma}
\begin{proof}
    Define the diagonal blocks
    \[S_L^{-1} \Sigma_L S_L^{-1} =
    \begin{pmatrix}
        D_{1,L} & 0 \\
        0 & D_{2,L}
    \end{pmatrix}.\]
    From Lemma \ref{lem:alignment (a) (b)} we have $D_{1,L} = \Sigma_C + o(1)$. Plugging into \ref{eq:C=SLXLSL-1} yields
    \[C^{r,r} = P_L^{r,r} D_{1,L} {P_L^{r,r}}^{\top} + \varepsilon_L D_{2,L} \varepsilon_L^{\top} = Q_L \Sigma_C Q_L^{\top} + \varepsilon_L D_{2,L} \varepsilon_L^{\top} + o(1).\]
    Thus it suffices to show that $\varepsilon_L D_{2,L} \varepsilon_L^{\top}$ converges to zero. This is an $r \times r$-matrix with coefficient in position $(i,j)$ given by
    \[(\varepsilon_L D_{2,L} \varepsilon_L^{\top})_{i,j} = \sum_{k=1}^{n-r} \frac{\sigma_{r+k}}{s_{i,L} s_{j,L}} (e_i^{\top} u_{r+k,L})(e_j^{\top} u_{r+k,L}).\]
    Notice that the terms $(e_i^{\top} u_{r+k,L})(e_j^{\top} u_{r+k,L})$ are bounded by 1, and that
    \[\frac{|\sigma_{r+k}|}{s_{i,L} s_{j,L}} \le \frac{|\sigma_{r+1}|}{\sigma_{r,L}} \cdot \frac{\sigma_{r,L}}{s_{r,L}^2} \xrightarrow{L \to \infty} 0\]
    according to Lemma \ref{lem:alignment (a) (b)}. Hence the $(i,j)$-coefficient of $\varepsilon_L D_{2,L} \varepsilon_L^{\top}$ converges to zero.
\end{proof}

\newpage
\section{Proofs of jacobian singular value dynamics}\label{app:proofs of sing val dynamics}
\subsection{Preliminaries}

\begin{lemma}
    \label{lem:dot_J}
    Consider an FGLN $J(t)=W_L(t) D_{L-1}\cdots D_1 W_1(t)$ trained by gradient flow on the weights $(W_\ell(t))_{\ell=1}^L$ with loss $\loss(J)$.
    Define $M_\ell(t) \coloneq D_\ell W_\ell(t)$ for $\ell=1,\dots,L$.
    Then:
    $$ \dot{J} = - \sum_{\l=1}^L M_{\l+1:L} D_\l M_{\l+1:L}^\top \nabla_J \loss(J) M_{1:\l-1}^\top M_{1:\l-1} $$
\end{lemma}

\begin{proof}
    $$\dot J = \sum_{\l=1}^L M_{\l+1:L} \dot{M}_\l M_{1:\l-1} $$   And for all $\l$, $ \dot{M}_\l = D_\l \dot W_\l = - D_\l \nabla_{W_\l}\loss$
    And using lemma \ref{lem:grad_W} we get the result.
\end{proof}

The derivative of singular values are related to the derivative of the Jacobian as follows.

\begin{lemma}
    \label{lem:dot_sk}
    Let $J$ a matrix depending smoothly on time.
    Let $J = USV^{\top}$ be its singular value decomposition.
    Then:
        $$u_k^\top\dot{J}v_k = \dot{s_k}$$
\end{lemma}

\begin{proof}
    \begin{equation*}
        \dot{J} = \dot{U}S V^\top + U \dot{S} V^\top + U S \dot{V}^\top
    \end{equation*}
    Hence
    \begin{equation*}
        u_k^\top\dot{J}v_k = u_k^\top\dot{u_k}s_k + \dot{s_k} + s_k \dot{v_k^\top}v_k
    \end{equation*}
    However $u_k^\top\dot{u_k} = \frac{1}{2}\frac{d}{dt}||u_k||^2_2 = 0$ and similarly $\dot{v_k^\top}v_k$ = 0 thus the result.
\end{proof}

\subsection{Proof of Proposition \ref{prop:singular_propto_fixed_gates}}

\begin{proof}[Proof of Proposition \ref{prop:singular_propto_fixed_gates}]
    \label{proof:prop_singular_propto_fixed_gates}
    We work at a fixed time $t$. Note that for all $\l$ we have $J = A_{\l} W_{\l} B_{\l}$. Applying Lemmas \ref{lem:dot_J} and \ref{lem:dot_sk} yields
    $$\dot s_k = -\left( \sum_{\l=1}^L U^{\top}A_{\l} A_{\l}^{\top} \nabla_J\loss \, B_{\l}^{\top} B_{\l} V\right)_{k,k}$$
    Incorporating our notations and assumptions, this becomes
    \[\dot s_k =-\sum_{\l =1}^L (U^{\top}U_{A_{\l}} S_{A_{\l}}^2 U_{A_{\l}}^{\top} U \left( U^{\top} \nabla_J \loss \, V \right) V^{\top} V_{B_{\l}} S_{B_{\l}}^2 V_{B_{\l}}^{\top} V)_{k,k}.\]
    Since we only care about the first order in $\varepsilon$ we may assume $U^{\top}U_{A_{\l}} = V^{\top} V_{B_{\l}} = I$ by assumption \emph{(ii)}. In this case the equation becomes
    \[\dot s_k \overset{\varepsilon \to 0}{\sim} -\sum_{\l =1}^L s_{k,A_{\l}}^2 s_{k,B_{\l}}^2 \left( U^{\top} \nabla_J \loss \, V \right)_{k,k}\]
    Assumption \emph{(i)} implies that $s_{k,A_{\l}} s_{k,B_{\l}} = e^{(1+\frac{1}{L})\delta_k} s_k^{1-\frac{1}{L}}$. Thus:
    $$\dot s_k \overset{\varepsilon \to 0}{\sim} -e^{(2+\frac{2}{L})\delta_k} s_{k}^{2 - \frac{2}{L}} \sum_{\l =1}^L\left( U^{\top} \nabla_J \loss \, V \right)_{k,k} = -e^{(2+\frac{2}{L})\delta_k} L s_{k}^{2 - \frac{2}{L}} \left( U^{\top} \nabla_J \loss \, V \right)_{k,k}$$
    The conclusion follows from the observation that
    \[\left( U^{\top} \nabla_J \loss \, V \right)_{k,k} = \inner{\nabla_{J} \loss}{u_k v_k^\top}.\]
\end{proof}

\subsection{Proof of Remark \ref{rem:singular propto}}

\begin{proof}[Proof of Remark \ref{rem:singular propto}]
    The derivative of the Jacobian reads
    \[\dot J = \sum_{\l = 1}^L A_{\l +1} \dot M_{\l} B_{\l-1}.\]
    Let $1\le k \le n$. By Lemma \ref{lem:dot_sk} we can access $\dot s_k$ via
    \[\dot s_k = (U^{\top} \dot J V)_{k,k} = \sum_{\l =1}^L \left( U^{\top} U_{A_{\l+1}} S_{A_{\l+1}} V_{A_{\l+1}}^{\top} \dot M_{\l} U_{B_{\l-1}} S_{B_{\l-1}} V_{B_{\l-1}}^{\top} V\right)_{k,k}.\]
    Since we only care about the first order in $\varepsilon$, by assumption \emph{(ii)} we may assume
    \[S_{A_{\l}}^{-1} U^{\top} U_{A_{\l}} S_{A_{\l}} = T_{\l}^-, \qquad S_{B_{\l}} V_{B_{\l}} V^{\top} S_{B_{\l}}^{-1} = T_{\l}^+.\]
    Then
    \begin{align*}
        \dot s_k &\overset{\varepsilon \to 0}{\sim} \sum_{\l =1}^L (S_{A_{\l+1}} T_{\l+1}^- V_{A_{\l+1}}^{\top} \dot M_{\l} U_{B_{\l-1}} T_{\l-1}^+S_{B_{\l}})_{k,k} \\
            & = \sum_{\l =1}^L s_{A_{\l+1},k}s_{B_{\l},k} (T_{\l+1}^- V_{A_{\l+1}}^{\top} \dot M_{\l} U_{B_{\l-1}} T_{\l-1}^+)_{k,k}
    \end{align*}
    Assumption \emph{(i)} ensures $s_{A_{\l+1},k}s_{B_{\l-1},k} = e^{\delta_k - \gamma_k}s_k = e^{(1+\frac{1}{L})\delta_k} s_k^{1-\frac{1}{L}}$ which is independent of $\l$ and can be factored out of the sum. This concludes the proof.
\end{proof}

\section{Diagonal correlation coefficient}\label{app:diagonal correlation}

The \emph{diagonal correlation coefficient}, denoted $\rho$, is a measure of how "diagonal" a square matrix is. It ranges from $-1$ to $1$, attains $1$ (resp. $-1$) precisely for diagonally-supported (resp. antidiagonally-supported) matrices.

Let $A$ be a nonzero $n \times n$ matrix with real entries; the computation of $\rho(A)$ is as follows. Up to taking absolute value we may assume that its coefficients are non-negative. We may now interpret $A_{ij}$ as the weight (or frequency) associated with the joint occurrence of indices $(i,j)$ in a random sampling of an $n \times n$ checkerboard. Define the coordinate random variables $(X,Y)$ with respect to this distribution:
\[\mathbb{P}(X=i, Y=j) = \frac{A_{ij}}{\sum_{k,\ell} A_{k\ell}}.\]
Then $\rho(A)$ is defined as the Pearson correlation coefficient between variables $X$ and $Y$:
\[\rho(A) = \frac{\mathrm{Cov}(X,Y)}{\sigma_X\sigma_Y}.\]
To express it purely in terms of the matrix $A$, we introduce
\[\mathbf{1} = (1,1,\dots,1)^\top, \qquad \mathbf{r_1} = (1,2,\dots,n)^\top, \qquad \mathbf{r_2} = (1^2,2^2,\dots,n^2)^\top .\]
Let $m = \mathbf{1}^\top A \mathbf{1}$ denote the total mass of $A$. With a little work, one can show that
\[\rho(A)= \frac
    {
    m \, (\mathbf{r_1}^\top A \mathbf{r_1})
    - (\mathbf{r_1}^\top A \mathbf{1})
    (\mathbf{1}^\top A \mathbf{r_1})
    }{
        \sqrt{
            \left(
            m \, (\mathbf{r_2}^\top A \mathbf{1})
            - (\mathbf{r_1}^\top A \mathbf{1})^2
            \right)
            \left(
            m \, (\mathbf{1}^\top A \mathbf{r_2})
            - (\mathbf{1}^\top A \mathbf{r_1})^2
            \right)
        }
    }.\]

\section{Additional Experiments}
\label{app:additional_xps}

\subsection{Singular value dynamics}
We train a FGLN of depth $L=10$, width $n=64$ and Bernoulli parameter $p=0.5$ on a random linear synthetic dataset. This dataset generates 1000 random input gaussian vectors and labels them through a fixed random linear map of rank 10. Both the input and output are of size $n$.

In this setup we compare the predicted top singular value trajectory versus the experimental one. The scaling factor in front of the prediction is fine-tuned to match the experiment, as it depends on several parameters that have not been considered in our study (such as the learning rate).
Denoting $g(t) = \left\langle \nabla_J \loss(J(t)),\, u_1(t) v_1(t)^\top \right\rangle$ where $u_1$ and $v_1$ are the first left and right singular vector of the full jacobian, we consider the iterative process: 
$$s(t+1) = s(t) + C s(t)^{2-\frac{2}{L}} g(t)$$
where $C$ is the fine tuned scaling factor. We observe in figure \ref{fig:singular_dynamics} that the predicted process matches the shape of the actual singular dynamics.

\begin{figure}[h]
    \centering
    \includegraphics[width=0.5\linewidth]{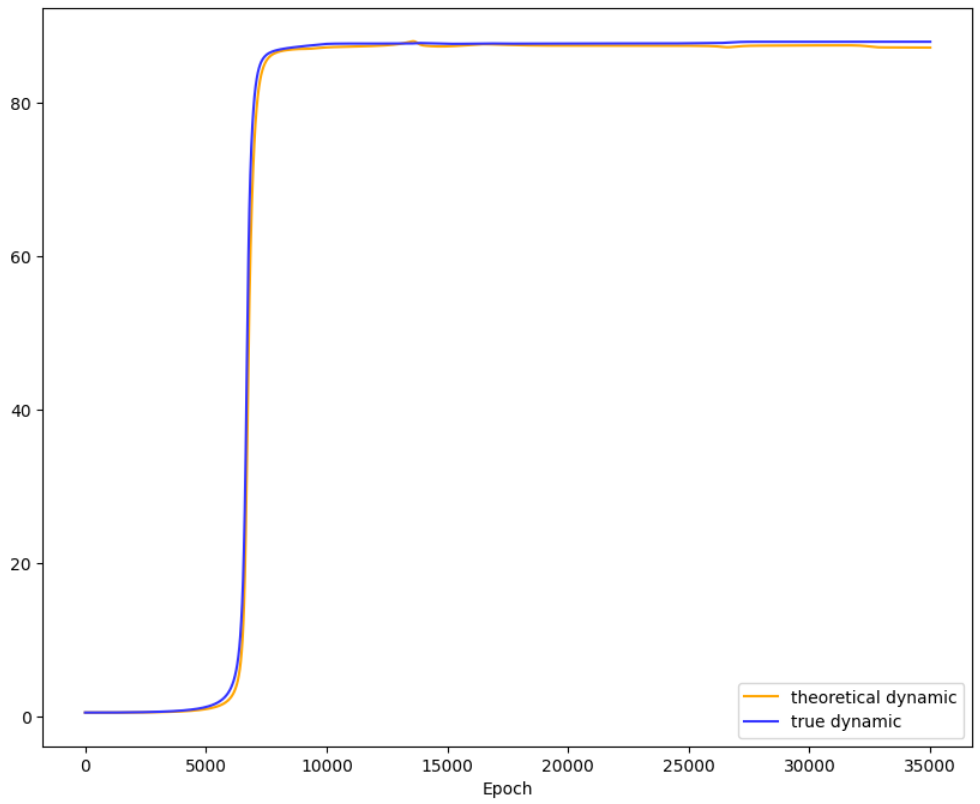}
    \caption{Theoretical versus experimental singular value dynamics for $s_1$.}
    \label{fig:singular_dynamics}
\end{figure}

\begin{figure}[h]
    \centering
    \includegraphics[width=0.5\linewidth]{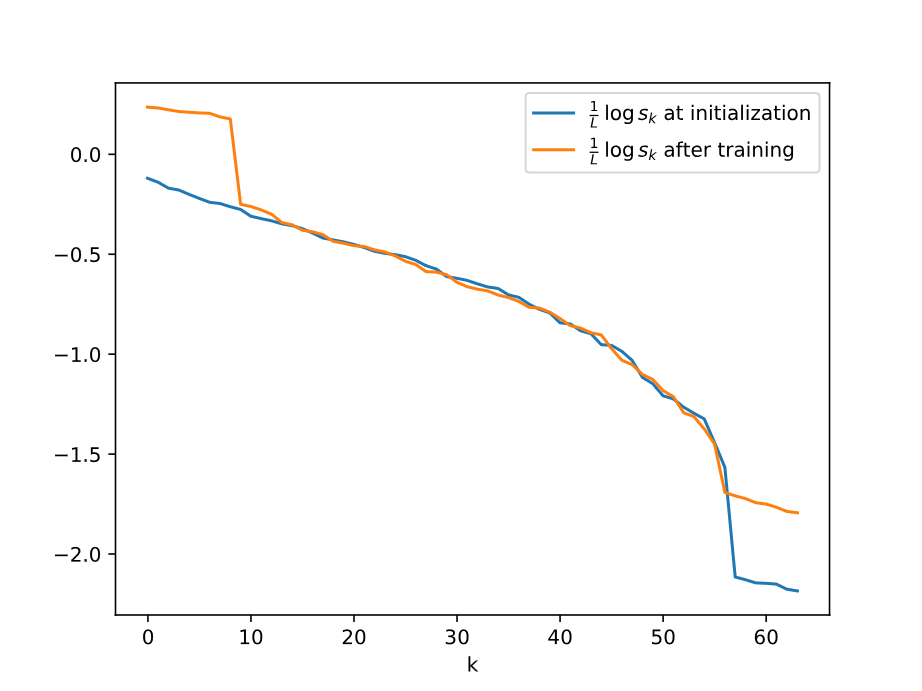}
    \caption{We compare the trained singular spectrum with respect to the one at its initialization. We observe that only the top singular values were strongly impacted by the training.}
\end{figure}

\subsection{Depth scaling and Alignment}

Here we consider a square FGLN at initialization with depth $L=20$ and varying Bernoulli parameter $p$ and width $n$.

\begin{figure}[h]
    \centering
    \includegraphics[width=0.4\linewidth]{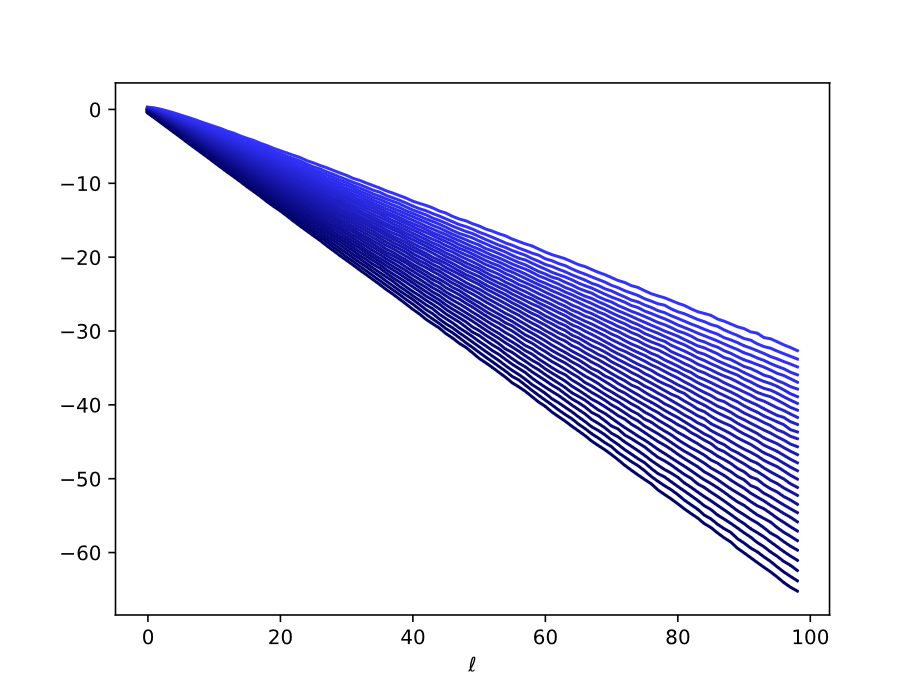}
    \includegraphics[width=0.4\linewidth]{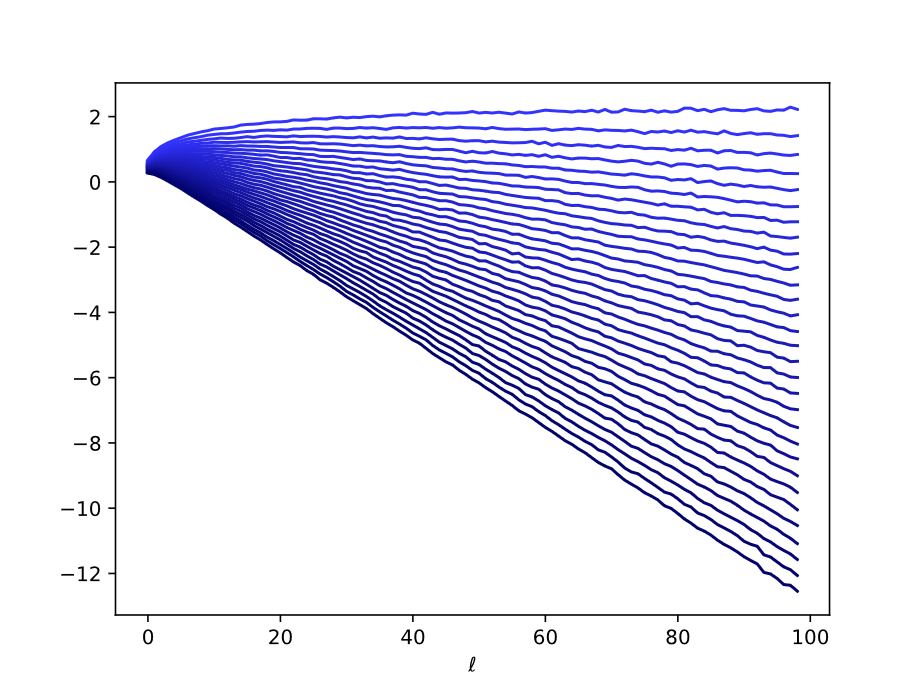}
    \caption{Top 30 singular value in an initialized FGLN with parameters $n=128$, $p=0.5$ (left) and $p=1$ (right)}
\end{figure}

We investigate the diagonal correlation of the product $U_{J_L}^{\top}A_{\l}A_{\l}^{\top}U_{J_L}$, which is found in the proof of Proposition \ref{prop:singular_propto_fixed_gates}. The clustering of $U_{J_L}^{\top}U_{A_{\l}}$ around the diagonal when $\l \ll L$ induces a similar effect for $U_{J_L}^{\top}A_{\l}A_{\l}^{\top}U_{J_L}$ by virtue of the SVD. However we see in Figure \ref{fig:diagonal_correlation_uaau_init} that diagonal clustering for this larger product also occurs to a lesser extent for $\l$ close to $L$.

\begin{figure}[h]
    \centering
    \includegraphics[width=0.34\linewidth]{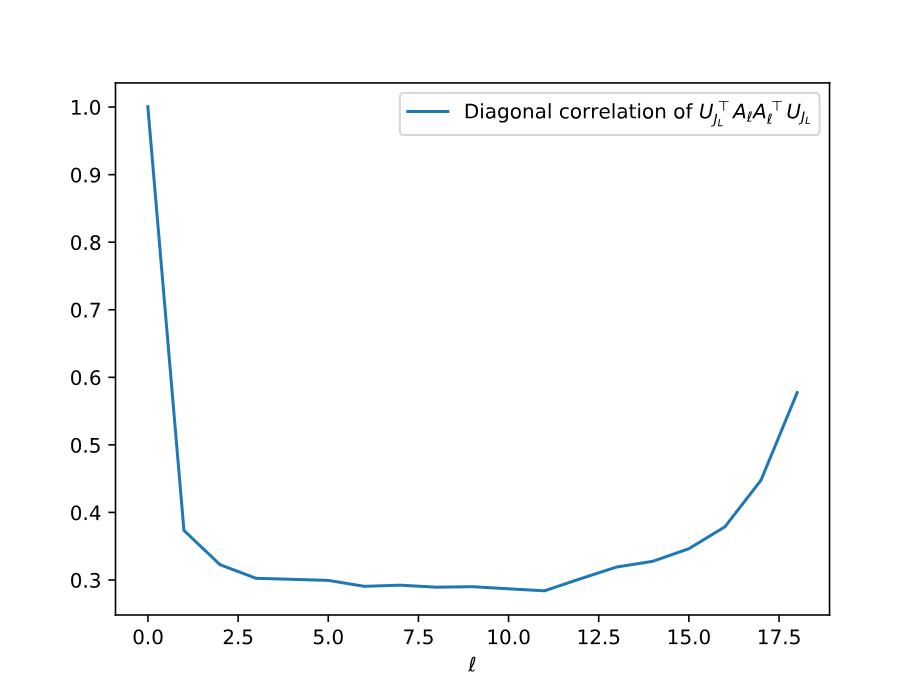}
    \includegraphics[width=0.34\linewidth]{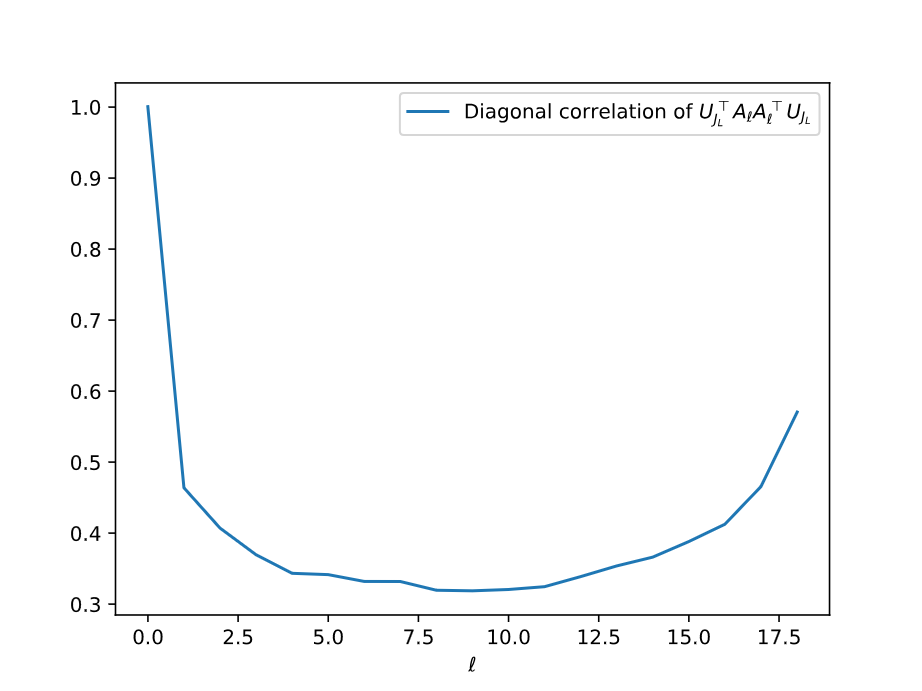}
    \includegraphics[width=0.34\linewidth]{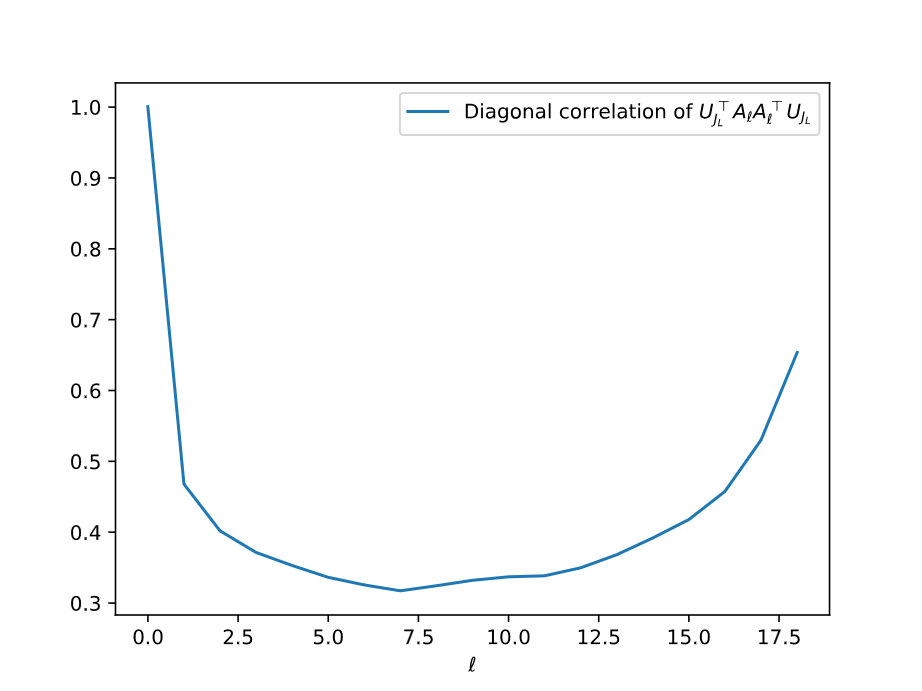}
    \includegraphics[width=0.34\linewidth]{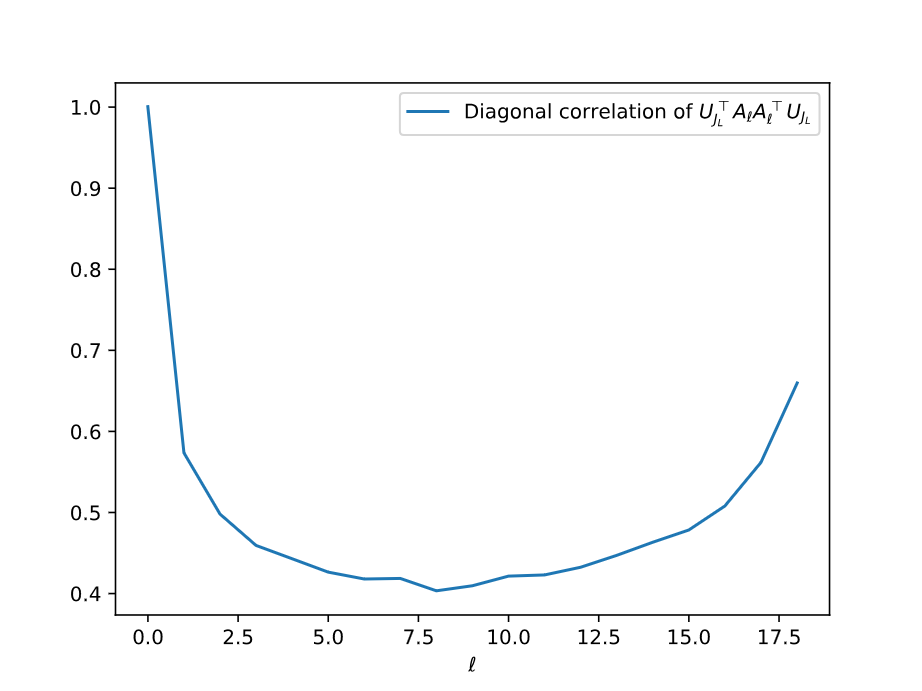}
    \caption{Diagonal correlations for the top-left $10 \times 10$ sub-matrix of $U_{J_L}^{\top} A_\l A_\l^{\top} U_{J_L}$ for parameter $p=0.5$ (first column), $p=1$ (second column), $n=64$ (first row) and $n=128$ (second row)}
    \label{fig:diagonal_correlation_uaau_init}
\end{figure}

This U-shaped profile is explained by two different mechanisms: the diagonal clustering of $U_{J_L}^{\top}U_{A_{\l}}$ when $\l \ll L$ (Figure \ref{fig:diagonal_correlation_uua_init}), as well as a diagonal clustering of $A_{\l}A_{\l}^{\top}$ for $\l \simeq L$ (Figure \ref{fig:diagonal_correlation_aa_init}). 

We interpret this second phenomenon as follows. We say that a matrix $M$ has \emph{isotropic spectrum} if its singular values $s_k$ are approximately equal. In the extreme case where $s_k = s$ for all $s$, we find $M M^{\top} = s^2 I$ to be perfectly diagonal; continuously altering the spectrum away from isotropy decreases the diagonal correlation of $MM^{\top}$. In our setting, as $\l$ decreases from $L$ to lower values the spectrum of $A_{\l} A_{\l}^{\top}$ exhibits stronger spectral separation, which explains Figure \ref{fig:diagonal_correlation_aa_init}.

\begin{figure}[h]
    \centering
    \includegraphics[width=0.34\linewidth]{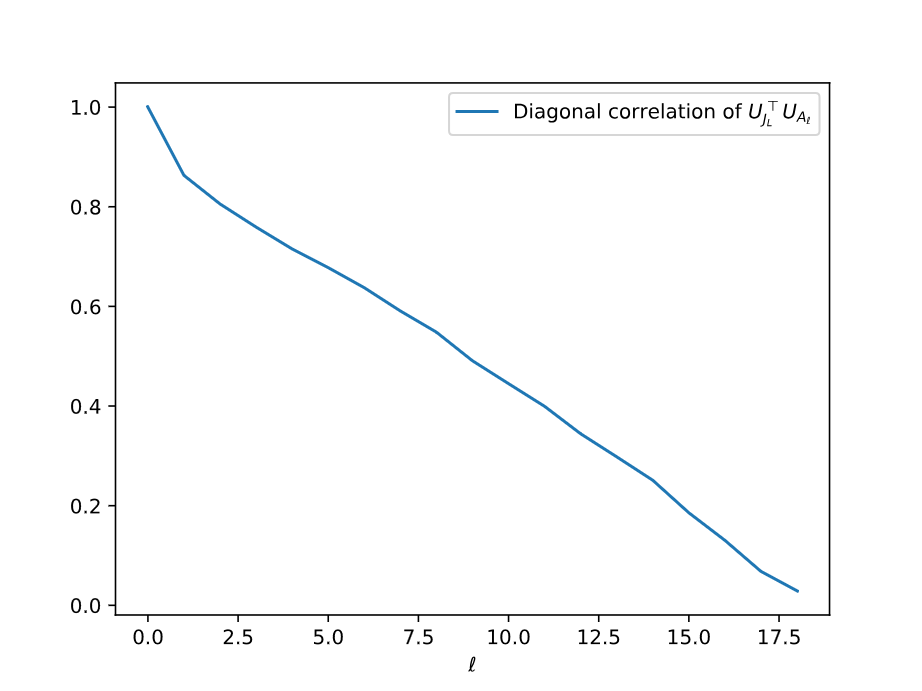}
    \includegraphics[width=0.34\linewidth]{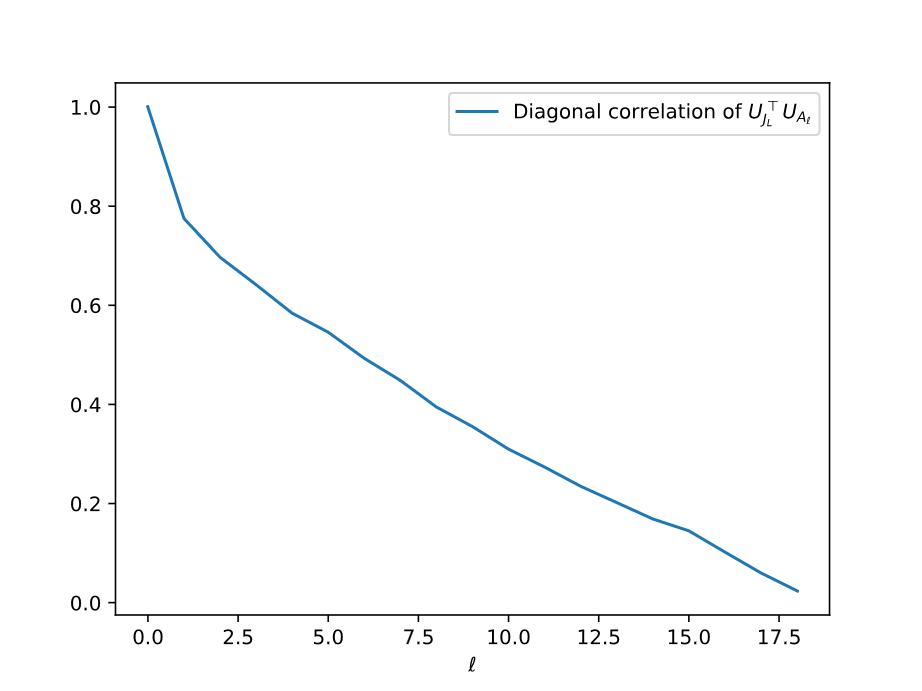}
    \includegraphics[width=0.34\linewidth]{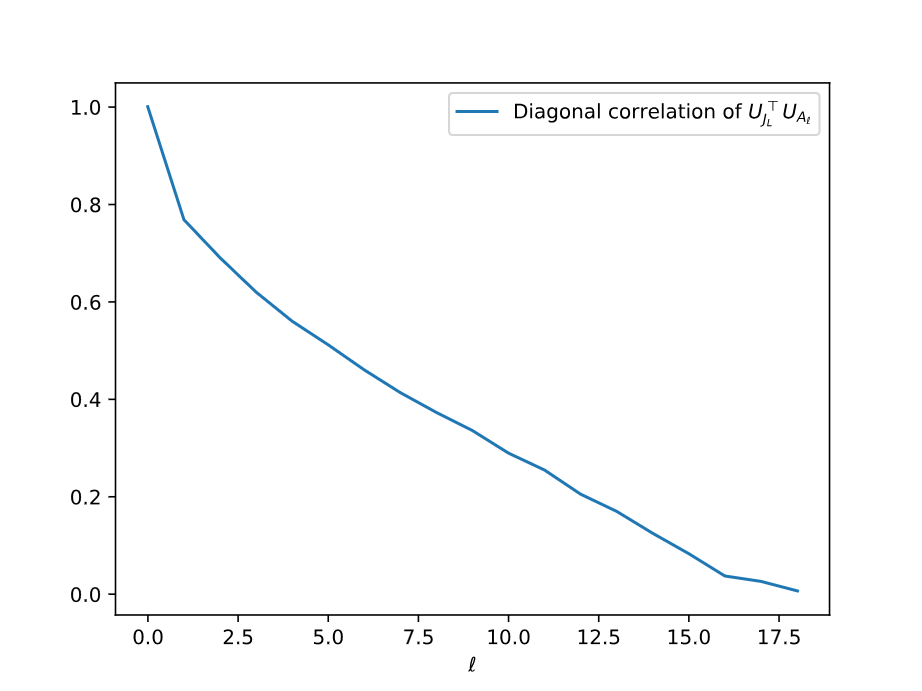}
    \includegraphics[width=0.34\linewidth]{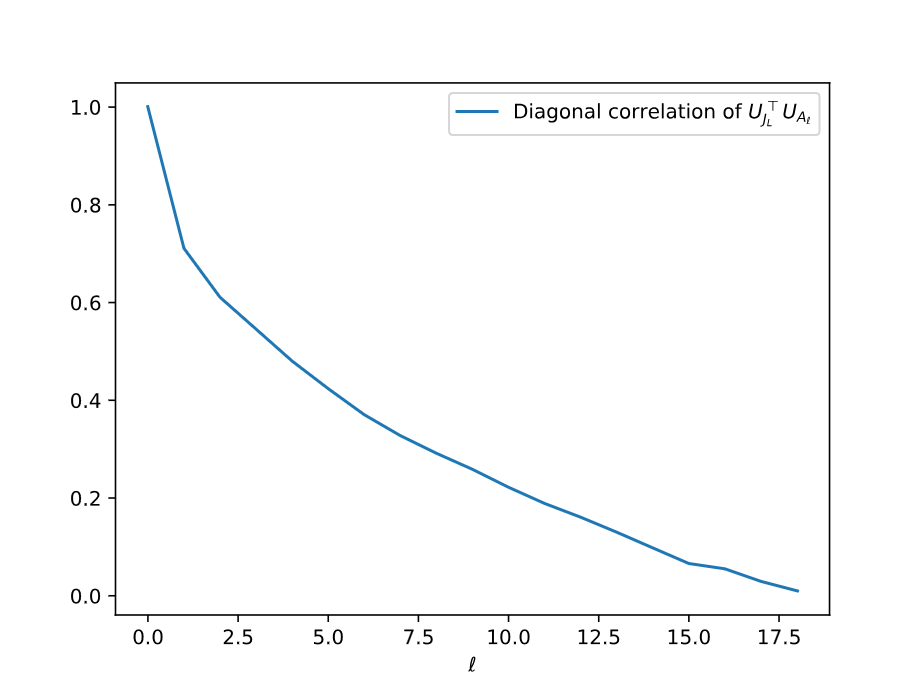}
    \caption{Diagonal correlations for the top-left $10 \times 10$ sub-matrix of $U_{J_L}^{\top} U_{A_\l}$ for parameter $p=0.5$ (first column), $p=1$ (second column), $n=64$(first row) and $n=128$(second row)}
    \label{fig:diagonal_correlation_uua_init}
\end{figure}

\begin{figure}[h]
    \centering
    \includegraphics[width=0.34\linewidth]{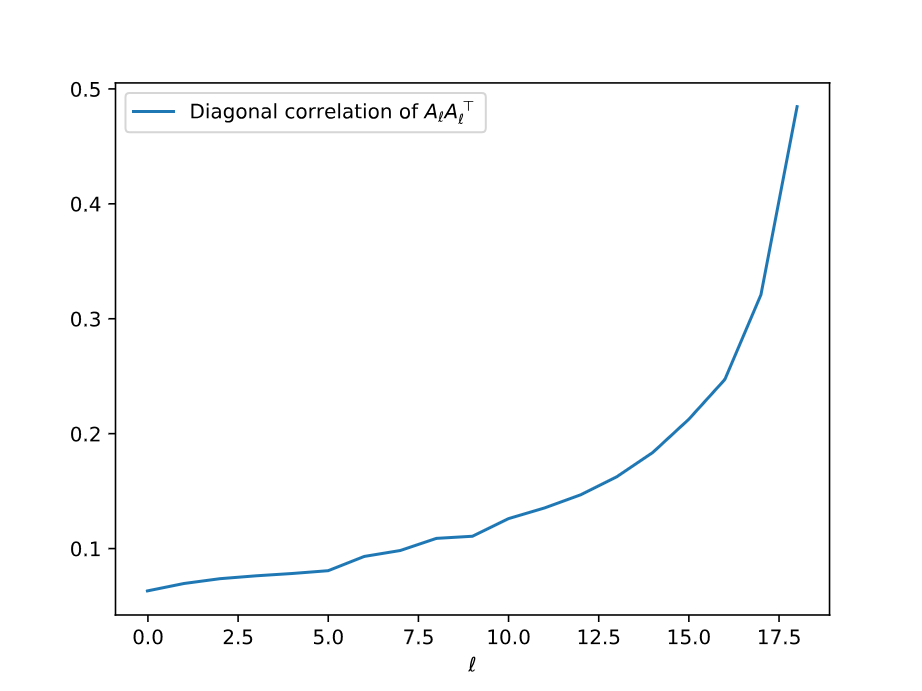}
    \includegraphics[width=0.34\linewidth]{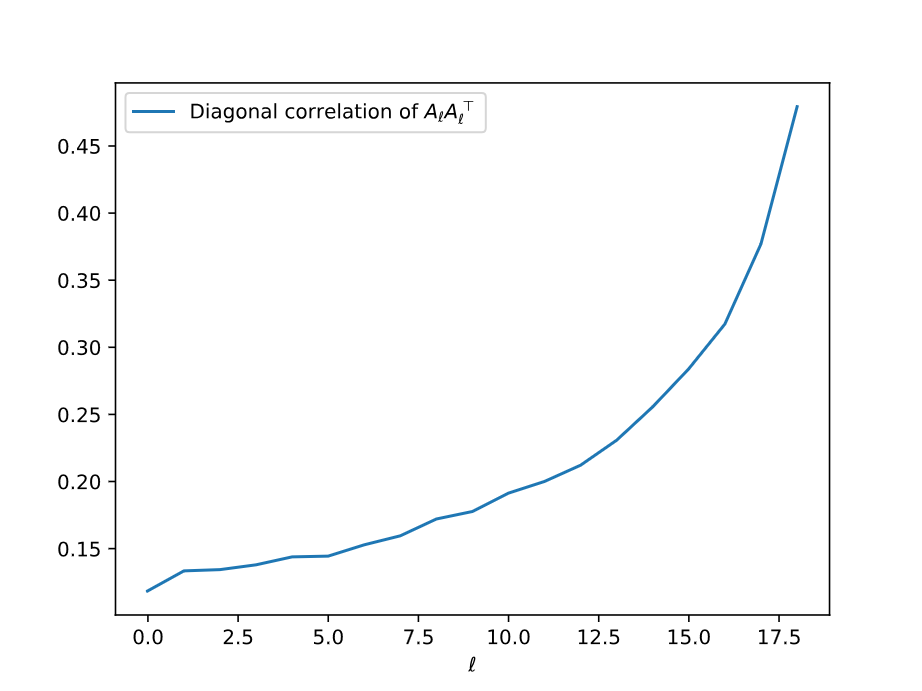}
    \includegraphics[width=0.34\linewidth]{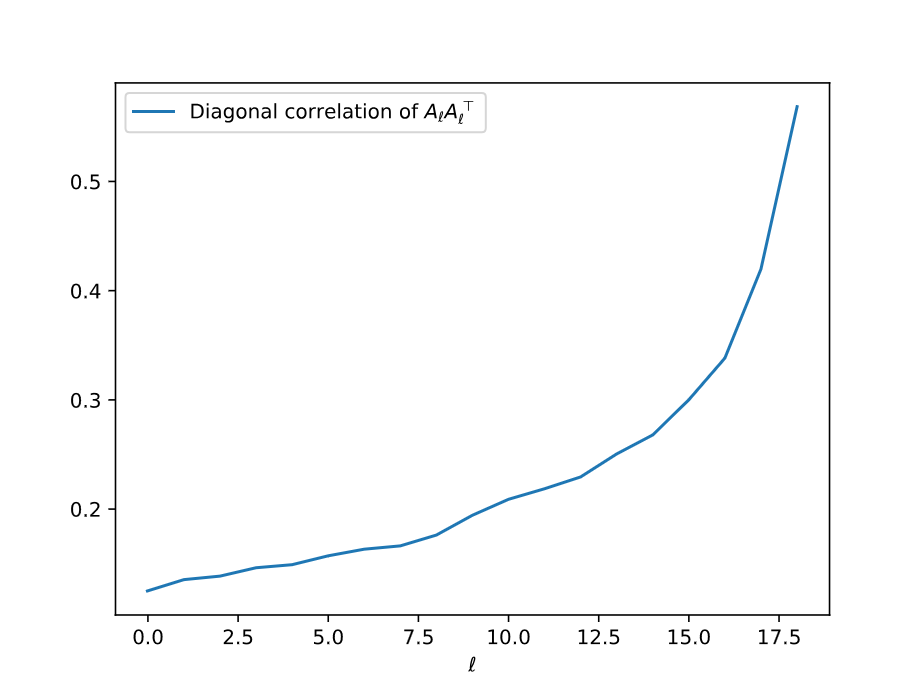}
    \includegraphics[width=0.34\linewidth]{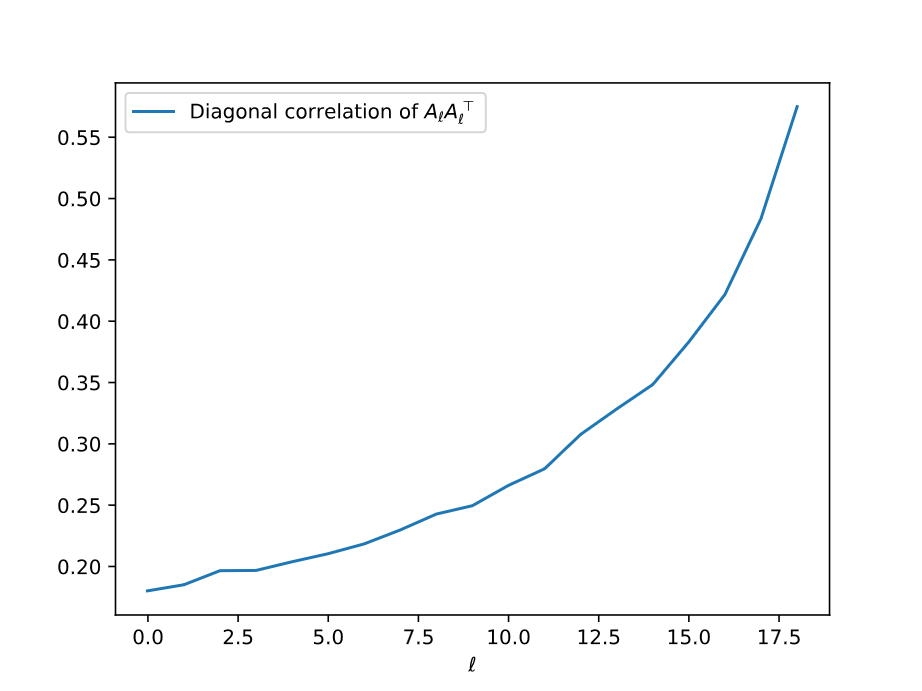}
    \caption{Diagonal correlations for the top-left $10 \times 10$ sub-matrix of $A_\l A_\l^{\top}$ for parameter $p=0.5$ (first column), $p=1$ (second column), $n=64$(first row) and $n=128$(second row)}
    \label{fig:diagonal_correlation_aa_init}
\end{figure}

%%%%%%%%%%%%%%%%%%%%%%%%%%%%%%%%%%%%%%%%%%%%%%%%%%%%%%%%%%%%%%%%%%%%%%%%%%%%%%%
%%%%%%%%%%%%%%%%%%%%%%%%%%%%%%%%%%%%%%%%%%%%%%%%%%%%%%%%%%%%%%%%%%%%%%%%%%%%%%%

\end{document}